\documentclass[10pt,journal]{IEEEtran}
\usepackage{array}
\usepackage[caption=false,font=footnotesize,labelfont=sf,textfont=sf,farskip=2pt,captionskip=2pt]{subfig}
\usepackage{textcomp}
\usepackage{stfloats}
\usepackage{verbatim}
\usepackage{default}
\usepackage{todonotes}
\usepackage{eso-pic}

\def\BibTeX{{\rm B\kern-.05em{\sc i\kern-.025em b}\kern-.08em
    T\kern-.1667em\lower.7ex\hbox{E}\kern-.125emX}}
\usepackage{balance}
\begin{document}
\title{Nonconvex Latent Optimally Partitioned Block-Sparse Recovery via Log-Sum and Minimax Concave Penalties}
\author{
    Takanobu Furuhashi, Hiroki Kuroda, Masahiro Yukawa, Qibin Zhao, Hidekata Hontani, and Tatsuya Yokota
    \thanks{This work was partially supported by the Japan Society for the Promotion of Science (JSPS) KAKENHI under Grant 23K28109.}
    \thanks{Takanobu Furuhashi, Hidekata Hontani and Tatsuya Yokota are with Nagoya Institute of Technology, Nagoya, Japan (e-mail: clz14116@nitech.jp, \{hontani.hidekata / t.yokota\}@nitech.ac.jp). Qibin Zhao and Tatsuya Yokota are also with RIKEN Center for Advanced Intelligence Project, Tokyo, Japan (e-mail: qibin.zhao@riken.jp). Hiroki Kuroda is with Nagaoka University of Technology, Niigata, Japan (e-mail: kuroda@vos.nagaokaut.ac.jp). Masahiro Yukawa is with Keio University, Yokohama, Japan (e-mail: yukawa@elec.keio.ac.jp).}
}

\markboth{Journal of \LaTeX\ Class Files,~Vol.~18, No.~9, September~2020}%
{How to Use the IEEEtran \LaTeX \ Templates}

\maketitle
\AddToShipoutPictureBG*{%
    \AtPageLowerLeft{%
        \makebox[\paperwidth]{%
            \parbox[b][1.5cm][c]{0.7\paperwidth}{%
                \centering
                \footnotesize
                This work has been submitted to the IEEE for possible publication. \\
                Copyright may be transferred without notice, after which this version may no longer be accessible.
            }%
        }%
    }%
}
\begin{abstract}
    We propose two nonconvex regularization methods, LogLOP-\ltwolone and AdaLOP-\ltwolone, for recovering block-sparse signals with unknown block partitions.
    These methods address the underestimation bias of existing convex approaches by extending log-sum penalty and the Minimax Concave Penalty (MCP) to the block-sparse domain via novel variational formulations.
    Unlike Generalized Moreau Enhancement (GME) and Bayesian methods dependent on the squared-error data fidelity term, our proposed methods are compatible with a broad range of data fidelity terms.
    We develop efficient Alternating Direction Method of Multipliers (ADMM)-based algorithms for these formulations that exhibit stable empirical convergence.
    Numerical experiments on synthetic data, angular power spectrum estimation, and denoising of nanopore currents demonstrate that our methods outperform state-of-the-art baselines in estimation accuracy.
\end{abstract}


\section{Introduction}
\IEEEPARstart{S}{parse} regularization is a fundamental technique in signal processing and machine learning, enabling the recovery of optimal solutions by exploiting the inherent sparsity of natural signals~\cite{10.2307/2346178,chenAtomicDecompositionBasis1998,zhangSurveySparseRepresentation2015,wenSurveyNonconvexRegularizationBased2018,johnwrightHighDimensionalDataAnalysis2022}.
We consider the problem of estimating a sparse signal $\bm x \in \mathbb{R}^N$ from observations $\bm y \in \mathbb{R}^J$ governed by the linear model $\bm y = \bm A \bm x + \bm \epsilon$, where $\bm A \in \mathbb{R}^{J \times N}$ is an observation matrix and $\bm \epsilon \in \mathbb{R}^J$ is an observation noise.

Sparse estimation is typically formulated as minimizing an objective function regularized by a sparsity-inducing penalty:
\begin{equation}
    \mathop{\mathrm{minimize}}_{\bm{x}\in\mathbb R^N} f(\bm{Ax}) + \lambda \Psi(\bm{Lx}),
    \label{eq:regularized_objective}
\end{equation}
where $f(\bm{Ax})$ is the data fidelity term (e.g., squared error $\|\bm y - \bm A \bm x\|_2^2/2$ for Gaussian noise, absolute error for Laplace noise, or I-divergence for Poisson noise).
The parameter $\lambda > 0$ balances fidelity and sparsity, while $\bm L \in \mathbb R^{K\times N}$ is a sparsifying transform, such as a Fourier or differential operator~\cite{heTimereassignedSynchrosqueezingTransform2019,sandryhailaDiscreteSignalProcessing2013,rodriguezTotalVariationRegularization2013,wangGuaranteedTensorRecovery2023,xueMultilayerSparsityBasedTensor2022}.
The most direct measure of sparsity is the $\ell_0$ pseudo-norm, $\norm{\bm x}_0$, which counts the number of nonzero elements. Nevertheless, the $\ell_0$ pseudo-norm is discontinuous and nonconvex, and its minimization is NP-hard~\cite{10.2307/2346178}.
Consequently, the $\ell_1$ norm, $\norm{\bm x}_1 = \sum_{n=1}^N \abs{x_n}$, is the standard convex surrogate for the $\ell_0$ pseudo-norm due to its computational tractability.
However, its convexity induces an inherent \emph{underestimation bias}, particularly for large-amplitude components, which can prevent the recovery of the true support~\cite{zhangNearlyUnbiasedVariable2010,lohSupportRecoveryIncoherence2017}.

\begin{figure}[tb]
    \centering
    \includegraphics[width=\linewidth]{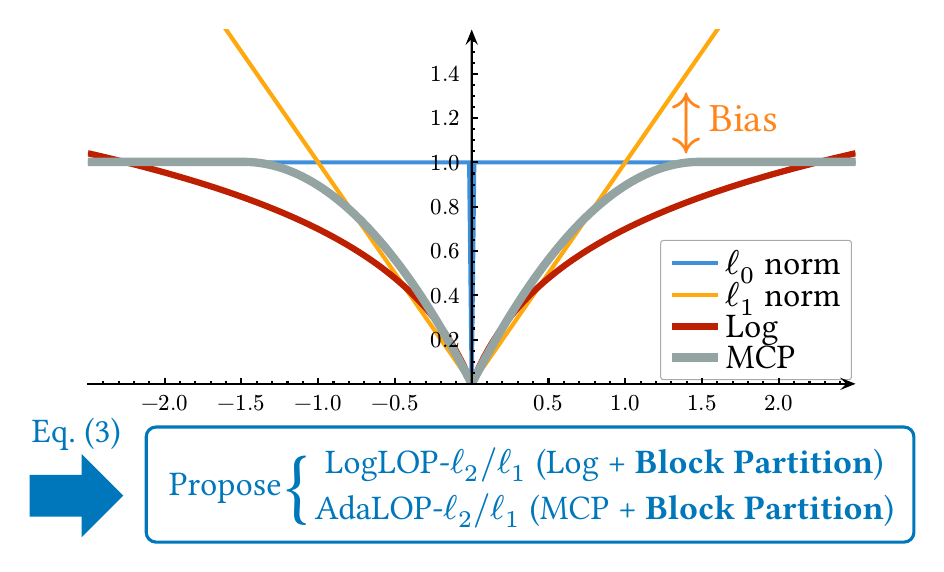}
    \caption{Comparison of sparsity-inducing penalties. Our proposed methods extend the concept of nonconvex log-sum and MC penalties by integrating a block partitioning mechanism to handle unknown block sparsity.}
    \zlabel{fig:compare_sparse_penalty}
\end{figure}

\emph{Nonconvex penalties} bridge the gap between the $\ell_0$ and $\ell_1$ penalties (\zcref{fig:compare_sparse_penalty}). By partially sacrificing convexity, these penalties facilitate nearly unbiased recovery while remaining computationally tractable.
The log-sum penalty~\cite{candesEnhancingSparsityReweighted2008} and Minimax Concave Penalty (MCP)~\cite{zhangNearlyUnbiasedVariable2010,selesnickSparseRegularizationConvex2017} are well-known examples that suppress the excessive penalty for large coefficients.
MCP can be interpreted as an optimally weighted $\ell_1$ penalty called Equivalent MCP (EMCP)~\cite{pokalaIterativelyReweightedMinimaxConcave2022,zhangTensorRecoveryBased2023}, using a weighting scheme whose efficacy has been well-established for assigning smaller weights to larger coefficients \cite{zhuIterativelyWeightedThresholding2021,zhaoAdaptiveWeightingFunction2024,sasakiSparseRegularizationBased2024}.

Although element-wise sparse regularizers are effective, natural signals often exhibit \emph{block sparsity}, where nonzero coefficients cluster into groups~\cite{yuanModelSelectionEstimation2006,jacobGroupLassoOverlap2009,eldarBlocksparseSignalsUncertainty2010,huangSelectiveReviewGroup2012,huGroupSparseOptimization2017,liStructuredSparseCoding2022,zhanLpLqMinimization2025}.
In practice, however, such partitions are seldom available a priori, rendering standard block-sparse regularizers ineffective.

To address this issue, Latent Optimally Partitioned (LOP)-based methods simultaneously estimate the signal and its latent block partitions~\cite{obozinskiGroupLassoOverlaps2011,fangPatterncoupledSparseBayesian2015,fangTwoDimensionalPatternCoupledSparse2016,8319524,heIterativelyReweightedMethod2016,fengLatentFusedLasso2020,santBlockSparseSignalRecovery2022,zhangBlockSparseBayesian2024a}.
Notably, the convex LOP-\ltwolone penalty~\cite{9729560,kurodaGraphStructuredSparseRegularization2022,furuhashiAdaptiveBlockSparse2024} enables efficient proximal optimization but inherits the underestimation bias.
Its nonconvex extension, Generalized Moreau Enhancement (GME)-LOP-\ltwolone penalty~\cite{kurodaConvexnonconvexFrameworkEnhancing2025}, mitigates this bias but is theoretically limited to least-squares problems (i.e., Gaussian noise assumptions).

Similarly, Sparse Bayesian Learning (SBL)~\cite{fangPatterncoupledSparseBayesian2015,santBlockSparseSignalRecovery2022} effectively recovers unknown block partitions but primarily relies on Gaussian noise for analytical tractability (e.g., conjugacy).
Moreover, SBL typically requires explicit matrix inversion, which is prohibitive in high-dimensional settings.

\begin{figure}[tb]
    \centering
    \includegraphics[width=0.9\linewidth]{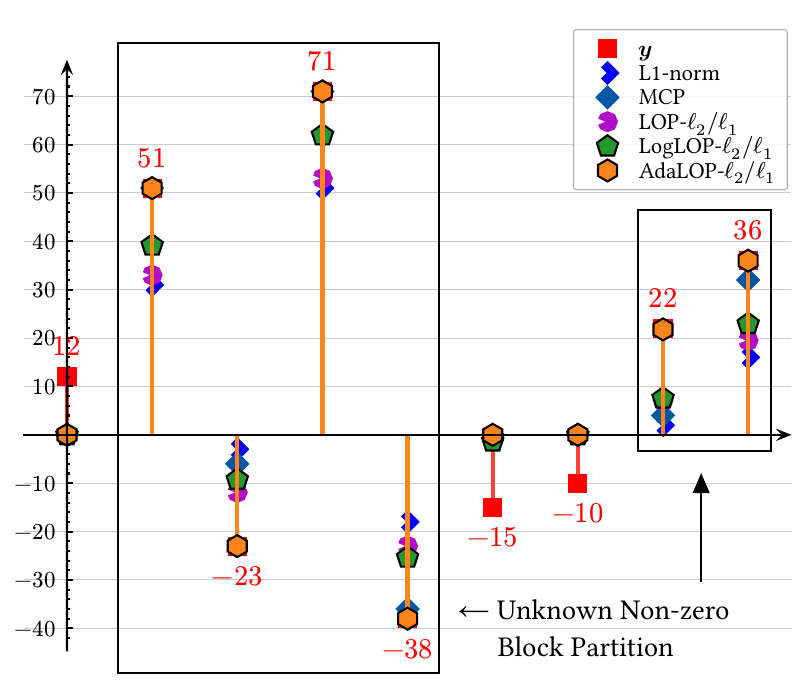}
    \caption{
        Comparison of proximal operators. Unlike LOP-\ltwolone, the proposed LogLOP-\ltwolone and AdaLOP-\ltwolone mitigate underestimation bias for large amplitudes (e.g., $y_2=51, y_4=71$) while preserving block partitions, outperforming element-wise penalties.
    }
    \zlabel{fig:compare_lopl2l1}
\end{figure}

This paper proposes two novel nonconvex LOP-based methods designed to enhance accuracy while preserving applicability in block-sparse signal recovery. As illustrated in \zcref{fig:compare_lopl2l1}, these methods significantly preserve large signal components better than convex approaches.
\begin{itemize}[itemsep=1ex,leftmargin=0cm]
    \item[] \textbf{Logarithmic LOP-\ltwolone (LogLOP-\ltwolone):} This method integrates a log-sum penalty into the LOP framework to mitigate underestimation within adaptively identified blocks. It is the first LOP-based approach to employ a logarithmic penalty for nearly unbiased recovery of block-sparse signals.    
    \item[] \textbf{Adaptively Weighted LOP-\ltwolone (AdaLOP-\ltwolone):} This method uses an adaptive weighting scheme inspired by EMCP. It updates weights according to the signal magnitude in identified blocks, better distinguishing significant coefficients.
\end{itemize}

The main contributions of this work are:
\begin{itemize}[leftmargin=*]
    \item We establish a general LOP-type Penalty framework unifying convex and the proposed nonconvex LOP-based regularizers. We prove that the \emph{asymptotically level stable} (als) property guarantees the existence of a minimizer.
    \item We propose LogLOP-\ltwolone and AdaLOP-\ltwolone penalties, the first methods to introduce logarithmic penalization and adaptive weighting into the LOP context. These approaches significantly reduce underestimation bias.
    \item    We develop efficient Alternating Direction Method of Multipliers (ADMM)-based algorithms for these formulations that exhibit stable convergence to the desired solutions empirically, despite the lack of theoretical guarantees due to nonconvexity (see the details in \zcref{rem:algs-convergence}).
    \item Unlike GME-LOP-\ltwolone and SBL-based methods, the proposed methods are compatible with \emph{any data fidelity term whose proximal operator is computable}~\cite{chauxVariationalFormulationFramebased2007, chierchiaProximityOperatorRepository2016}, extending their utility to diverse noise models and applications.
\end{itemize}

The remainder of this paper is organized as follows. \zcref{sec:preliminaries} reviews existing methods. \zcref{sec:proposed} details the proposed LogLOP-\ltwolone and AdaLOP-\ltwolone penalties. \zcref{sec:optimization} describes the optimization algorithms. \zcref{sec:experiments} presents numerical results, and \zcref{sec:conclusion} concludes the paper.

\section{Review of Existing Methods}
\zlabel{sec:preliminaries}

\subsection{Mathematical Notations}
Vectors and matrices are denoted by bold lowercase ($\bm a$) and uppercase ($\bm A$) letters, respectively. $\bm x \in \R^N$ denotes an $N$-dimensional vector with $n$-th component $x_n$, and $\bm X \in \R^{I \times J}$ denotes an $I \times J$ matrix with the $(i,j)$ element $X_{i,j}$. The transpose and Hermitian transpose are denoted by $(\cdot)^\top$ and $(\cdot)^\mathsf{H}$, respectively.
$\mathbb N$, $\mathbb R$, $\mathbb R_{>0}$, and $\mathbb R_{\geq 0}$ represent the sets of positive integers, real numbers, positive reals, and nonnegative reals, respectively. $\bm O$ and $\bm I$ denote the zero matrix and identity matrix, respectively.
The set $[N]$ denotes $\{1, \ldots, N\}$, and $\abs{S}$ is the cardinality of set $S$. For a vector $\bm x$ and index set $S$, $\bm x_S$ denotes the subvector indexed by $S$.
$\abs{x}$ is the absolute value, and $\sign(x)$ is the sign function. $\abs{\bm x}$ denotes the element-wise absolute vector. The $\ell_1$ and $\ell_2$ norms are defined as $\norm{\bm x}_1 = \sum_{n=1}^N \abs{x_n}$ and $\norm{\bm x}_2 = (\sum_{n=1}^N x_n^2)^{1/2}$. $\bm x \odot \bm y$ denotes the element-wise product. We also use $\bm w \odot h$ for a vector $\bm w$ and a scalar function $h$ to denote the function defined by $(\bm w \odot h)(\bm x) = \sum_{n=1}^N w_n h(x_n)$.
A function $h: \mathbb{R}^N \to \mathbb{R} \cup \{\infty\}$ is \emph{proper} if its domain $\dom(h) \coloneqq \{ \bm x \in \mathbb{R}^N \mid h(\bm x) < \infty \}$ is nonempty. It is \emph{lower semicontinuous} (lsc) if its lower level set $\{ \bm x \in \mathbb{R}^N \mid h(\bm x) \leq a \}$ is closed for every $a \in \mathbb{R}$. It is \emph{coercive} if $\lim_{\|\bm x\|_2\to\infty} h(\bm x) = \infty$.
The proximal operator of a proper, lsc convex function $h$ with $\beta > 0$ is
\begin{equation}
    \prox_{\beta h}(\bm v) \coloneqq \arg\min_{\bm x \in \R^N} \left\{ h(\bm x) + \frac{1}{2\beta} \norm{\bm x - \bm v}_2^2 \right\}.
\end{equation}

\subsection{Nonconvex Sparse Regularization}
Although the $\ell_1$ norm is a standard tool for sparse estimation, its convexity and coercivity induces an underestimation bias for large coefficients. Nonconvex regularizations provide a more accurate measure of sparsity~\cite{fanVariableSelectionNonconcave2001,strekalovskiyRealTimeMinimizationPiecewise2014,sasakiSparseRegularizationBased2024}.

The log-sum penalty~\cite{candesEnhancingSparsityReweighted2008} mitigates this bias:
\begin{equation}
    \Omega_{\epsilon}(\bm x) = \sum_{n=1}^N \log\left(\frac{\abs{x_n}}{\epsilon}+1\right),
\end{equation}
where $\epsilon > 0$ is a curvature parameter.
It approaches the $\ell_0$ pseudo-norm as $\epsilon \to 0$ and the $\ell_1$ norm as $\epsilon \to \infty$.
Its logarithmic growth penalizes large amplitudes less severely than the $\ell_1$ norm.                                 However, because the penalty is coercive and not upper bounded, the bias is reduced but not eliminated.

The Minimax Concave Penalty (MCP)~\cite{zhangNearlyUnbiasedVariable2010} addresses this by employing a bounded penalty:
\begin{equation}
    \Omega_{\gamma, w}(x) = \begin{cases}
        w \abs{x} - \frac{x^2}{2\gamma}, & \text{if } \abs{x} \leq \gamma w; \\
        \frac{\gamma w^2}{2},            & \text{if } \abs{x} > \gamma w,
    \end{cases}
\end{equation}
where $w > 0$ is the sparsity parameter and $\gamma > 1$ is the curvature parameter. It approaches the $\ell_0$ ($\gamma \to 1$) and $\ell_1$ ($\gamma \to \infty$) norms. By becoming constant for $\abs{x} > \gamma w$, MCP achieves unbiasedness for large coefficients.
MCP can be extended to vectors as $\Omega_{\gamma, \bm w}(\bm x) = \sum_{n} \Omega_{\gamma, w_n}(x_n)$ with $\bm w \in \mathbb{R}_{\geq 0}^N$.

Crucially, MCP admits a variational representation known as the vector Equivalent MCP (EMCP)~\cite{pokalaIterativelyReweightedMinimaxConcave2022}:
\begin{equation}
    \Omega_{\gamma, \bm w}(\bm x) = \min_{\bm \omega \in \mathbb R_{\geq 0}^N} \left\{\norm{\bm x}_{\bm \omega,1} + \frac{\gamma}{2} \norm{\bm \omega - \bm w}_2^2\right\},
    \label{eq:vector-emcp}
\end{equation}
where $\norm{\bm x}_{\bm \omega,1} = \sum_n \omega_n \abs{x_n}$. This interprets MCP as an optimally weighted $\ell_1$ penalty. The method in~\cite{pokalaIterativelyReweightedMinimaxConcave2022} proposes to update $\bm w$ adaptively to enhance the estimation accuracy. This adaptive capability is central to our proposed method.

\subsection{Latent Optimal Partitioned (LOP)-\ltwolone}
Standard block-sparse regularization relies on given block partitions. To address scenarios where the partitions are unknown, the convex LOP-\ltwolone penalty~\cite{9729560} simultaneously estimates the signal and its optimal block partitions.

The method aims to identify a block partitions $(\mathcal B_k)_{k=1}^j$ that minimizes the convex mixed \ltwolone norm:
\begin{equation}
    \psi_K(\bm x) \coloneqq \min_{j \in [K]}\left[
        \min_{(\mathcal B_k)_{k=1}^j\in\mathcal P_j}
        \sum_{k=1}^j \sqrt{\abs{\mathcal B_k}} \norm{\bm x_{\mathcal B_k}}_2
        \right],
    \label{eq:minimize-mixedl2l1norm}
\end{equation}
where $\mathcal P_j$ consists of all possible nonoverlapping block partitions with $j$ blocks~\cite[Eq. 3]{9729560}.
Because this combinatorial problem is intractable, LOP-\ltwolone employs a tight convex relaxation based on the variational form of the $\ell_2$ norm. The variational function $\phi: \mathbb R \times \mathbb R \to \mathbb R_{\geq 0} \cup \{\infty\}$ is defined as:
\begin{equation}
    \phi(x,\tau) \coloneqq \begin{cases}
        \displaystyle \frac{\abs{x}^2}{2\tau} + \frac{\tau}{2}, & \text{if } \tau > 0;                    \\
        0,                                                      & \text{if } x = 0 \text{ and } \tau = 0; \\
        \infty,                                                 & \text{otherwise}.
    \end{cases}
    \label{eq:var-l2norm}
\end{equation}
Using this, we define the LOP-\ltwolone penalty as follows:
\begin{equation}
    \Psi_\alpha(\bm x) \coloneqq \min_{\substack{\bm\sigma\in\R^{N} \\
            \norm{\bm{D\sigma}}_1 \leq \alpha}} \sum_{n=1}^N \phi(x_n,\sigma_n).
    \label{eq:lop-l2l1-penalty}
\end{equation}
The latent vector $\bm\sigma$ represents the block partitions, constrained by Total Variation (TV) $\norm{\bm{D\sigma}}_1 \leq \alpha$. The parameter $\alpha \geq 0$ controls the partition granularity, and $\bm D$ denotes the first-order difference matrix.
Although LOP-\ltwolone effectively discovers the block partitions, it inherits the convexity, leading to the same \emph{underestimation bias} for large amplitudes (\zcref{fig:compare_lopl2l1}).

\subsection{Generalized Moreau Enhancement (GME)-LOP-\ltwolone}
To mitigate this bias, GME-LOP-\ltwolone~\cite{kurodaConvexnonconvexFrameworkEnhancing2025} applies the Generalized Moreau Enhancement (GME) framework~\cite{lanzaSparsityInducingNonconvexNonseparable2019,abeLinearlyInvolvedGeneralized2020,al-shabiliSharpeningSparseRegularizers2021} to LOP-\ltwolone. It constructs a nonconvex penalty $\Psi_{\mathrm{GL}}(\bm{x})$ by subtracting the generalized Moreau envelope of $\Psi_{\alpha}(\bm{x})$:
\begin{equation}
    \Psi_{\mathrm{GL}}(\bm{x}) = \Psi_{\alpha}(\bm{x}) - \min_{\bm{v} \in \mathbb{R}^N} \left\{ \Psi_{\alpha}(\bm{v}) + \frac{1}{2} \norm{\bm{B}(\bm{x}-\bm{v})}_2^2 \right\}.
    \label{eq:gme-lop-l2l1}
\end{equation}
A key feature of GME-LOP-\ltwolone is that, under specific conditions on the GME parameter matrix $\bm B$ (e.g., $\bm{A}^\top\bm{A} - \lambda\bm{L}^\top\bm{B}^\top\bm{BL} \succeq \bm{O}$), the overall regularized least-squares objective remains convex. Consequently, its convergence to the global minimum is guaranteed.

However, this convexity guarantee is strictly tied to the squared error loss. This restricts its applicability to other data fidelity terms, motivating the need for the more flexible nonconvex LOP-based methods proposed in this work. Note that while our proposed methods do not possess the same theoretical convergence guarantee as GME-LOP-\ltwolone (see \zcref{rem:algs-convergence}), they exhibit stable convergence behavior empirically, even with nonsquared error losses.
\section{Design of Proposed Regularizations}
\zlabel{sec:proposed}

Existing LOP-based methods face a dilemma: the convex LOP-\ltwolone penalty suffers from underestimation bias, while the nonconvex GME-LOP-\ltwolone penalty is restricted to squared error loss.
To overcome these limitations, we propose two novel nonconvex frameworks that combine the block-discovery capability of LOP-\ltwolone with enhanced bias mitigation, without restricting the data fidelity term.

\begin{itemize}[leftmargin=0cm, topsep=0.5ex, itemsep=0.5ex] \sloppy
    \item[] \textbf{Logarithmic LOP-\ltwolone (LogLOP-\ltwolone):} Mitigates underestimation by replacing the linear penalty on block magnitude with a logarithmic function.
    \item[] \textbf{Adaptively Weighted LOP-\ltwolone (AdaLOP-\ltwolone):} Jointly optimizes block partitions and regularization weights to automatically down-weight significant components.
\end{itemize}

As shown in \zcref{fig:asymptotic_behaviors}, the asymptotic forms of both LogLOP-\ltwolone and AdaLOP-\ltwolone (discussed in \zcref{thm:var-rwl1, thm:ada-lop-l2l1}) exhibit clear nonconvex characteristics. The penalty growth diminishes for large signal amplitudes, which mitigates the estimation bias compared to the convex LOP-\ltwolone penalty.

\begin{figure}[tb]
    \centering
    \subfloat[LOP ($\alpha \to 0$)]{%
        \includegraphics[width=0.31\linewidth]{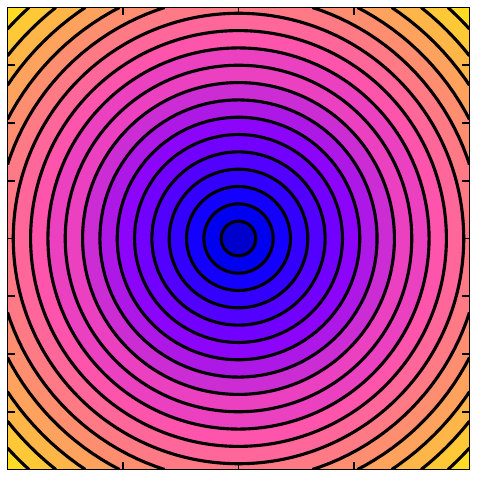}%
        \zlabel{fig:lop_alpha_0}
    }\hfill
    \subfloat[LogLOP ($\alpha \to 0$)]{%
        \includegraphics[width=0.31\linewidth]{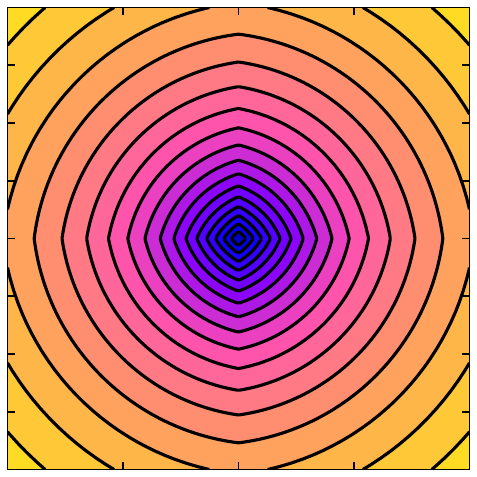}%
        \zlabel{fig:relop_alpha_0}
    }\hfill
    \subfloat[AdaLOP ($\alpha \to 0$)]{%
        \includegraphics[width=0.31\linewidth]{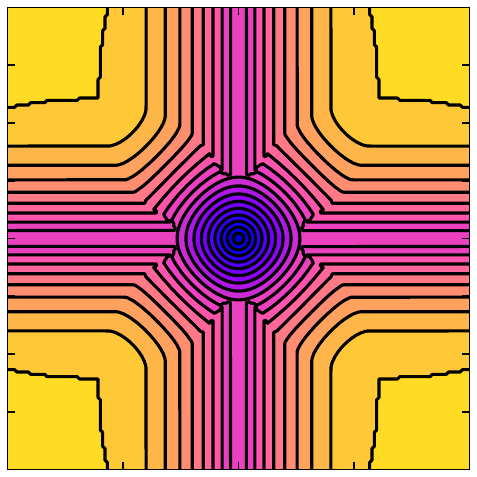}%
        \zlabel{fig:adalop_alpha_0}
    }\\
    \subfloat[LOP ($\alpha \to \infty$)]{%
        \includegraphics[width=0.31\linewidth]{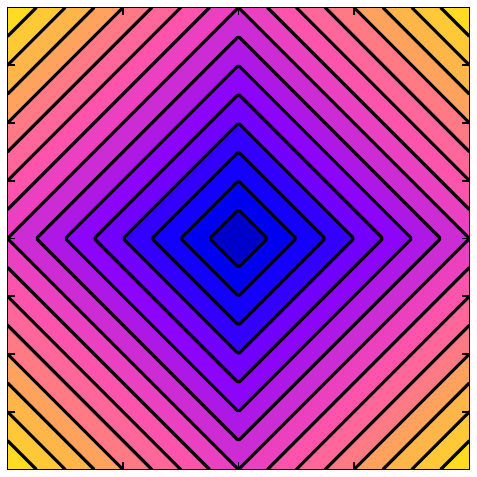}%
        \zlabel{fig:lop_alpha_inf}
    }\hfill
    \subfloat[LogLOP ($\alpha \to \infty$)]{%
        \includegraphics[width=0.31\linewidth]{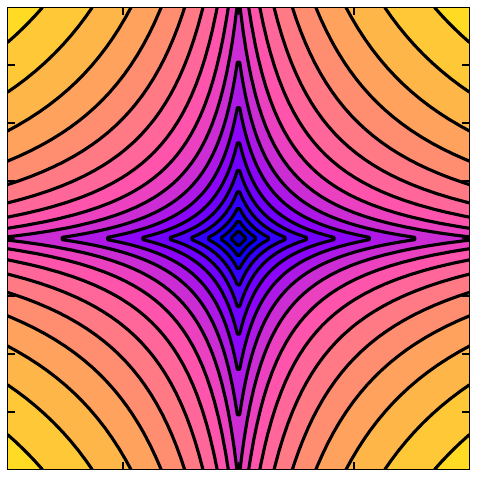}%
        \zlabel{fig:relop_alpha_inf}
    }\hfill
    \subfloat[AdaLOP ($\alpha \to \infty$)]{%
        \includegraphics[width=0.31\linewidth]{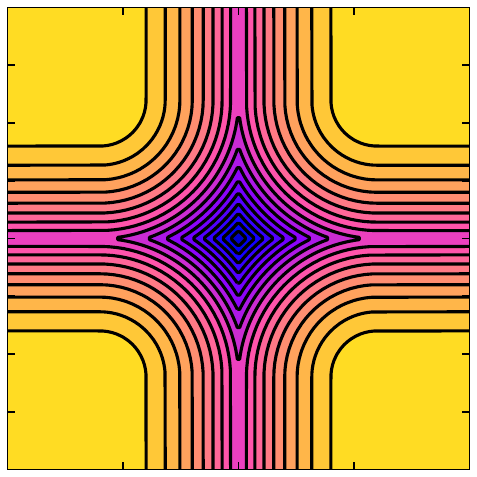}%
        \zlabel{fig:adalop_alpha_inf}
    }
    \caption{Asymptotic behaviors of LOP, LogLOP (\zcref{thm:var-rwl1}), and AdaLOP (\zcref{thm:ada-lop-l2l1}) variants for \ltwolone regularization.}
    \zlabel{fig:asymptotic_behaviors}
\end{figure}

\subsection{General Variational Framework for Optimal Partitioning}
To systematically address the bias issue, we first abstract the LOP-\ltwolone mechanism into a general variational framework.
The core idea of LOP-\ltwolone is to evaluate signal components $x_n$ via a latent variable $\sigma_n$ that captures the optimal block partitions. Unlike GME-LOP-\ltwolone , which relies on GME for non-convexification, we define a general LOP-type penalty $\Psi_{\alpha, \varphi}$ by directly employing a non-convex variational function $\varphi: \R \times \R \to \R_{\geq 0} \cup \{\infty\}$ as follows.

\begin{definition}[General LOP-type Penalty]
    \zlabel{def:general-lop-type-penalty}
    Let $\varphi: \R \times \R \to \R_{\geq 0} \cup \{\infty\}$ be a proper and lsc function. Assume that $\varphi$ satisfies the \emph{asymptotically level stable} (als) property~\cite[Def. 3.3.2]{alfredauslenderAsymptoticConesFunctions2003}.
    For a given partition parameter $\alpha \ge 0$, the general LOP-type penalty $\Psi_{\alpha, \varphi}: \R^N \to \R_{\geq 0} \cup \{\infty\}$ is defined as:
    \begin{equation}
        \Psi_{\alpha, \varphi}(\bm x) \coloneqq \min_{\substack{\bm\sigma\in\R^{N} \\ \norm{\bm{D\sigma}}_1 \leq \alpha}} \sum_{n=1}^N \varphi(x_n, \sigma_n).
        \label{eq:general-lop-framework}
    \end{equation}
\end{definition}

The als property characterizes the stability of level sets at infinity and generalizes the assumption of coercivity, ensuring the existence of a minimizer for $\Psi_{\alpha, \varphi}(\bm x)$ (see \zcref{app:existence_min_adalop} for details). Furthermore, it ensures that a minimizer $(\hat{\bm x}, \hat{\bm \sigma})$ exists when minimizing an objective (e.g., \zcref{eq:min-relop-l2l1}) composed of this penalty and a proper, lsc, coercive and convex function $f$.

The variational function $\varphi(x, \sigma)$ is designed as a variational upper bound of a base penalty $\Omega(x)$ (e.g., $\ell_1$ norm~\cite[Theorem 2]{9729560}, log-sum and MC penalty), satisfying $\min_{\sigma \in \mathbb{R}} \varphi(x, \sigma) = \Omega(x)$. Minimizing this upper bound under the TV constraint allows the framework to discover the optimal block partitions while applying the properties of $\Omega(x)$.
LOP-\ltwolone uses the function in \eqref{eq:var-l2norm} that produces the absolute penalty $\abs{x}$.
To reduce estimation bias, LogLOP-\ltwolone and AdaLOP-\ltwolone employ nonconvex variational functions satisfying the als condition, ensuring the well-definedness of the penalties and the existence of minimizers.

\subsection{Logarithmic LOP-\ltwolone (LogLOP-\ltwolone)}
\subsubsection{Variational Formulation}
We realize LogLOP-\ltwolone by choosing a variational function $\phi_\epsilon$ that induces logarithmic growth. We define $\phi_{\epsilon}:\mathbb R \times \mathbb R \to \mathbb R_{\geq 0} \cup \{\infty\}$ with $\epsilon > 0$:
\begin{equation}
    \phi_\epsilon(x,\tau) \coloneqq \left\{ \begin{array}{@{\!\;}lr}
        \displaystyle \frac{1}{2\tau} \left(\frac{\abs{x}}{\epsilon}+1\right)^2 + \frac12\log\tau -\frac12, & \text{if } \tau \geq 1; \\
        \infty,                                                                                             & \text{otherwise}.
    \end{array} \right.
    \label{eq:var-rwl1}
\end{equation}
We design this function such that its minimum over $\tau$ recovers the logarithmic penalty, as proved in \zcref{lem:var-rwl1}.
\begin{lemma}[Minimum Value of $\phi_\epsilon(x,\tau)$]
    \zlabel{lem:var-rwl1}
    For any fixed $x \in \mathbb R$ and $\epsilon > 0$, $\phi_\epsilon(x,\tau)$ has a unique global minimizer $\tau^* = (\abs{x}/\epsilon + 1)^2$ over $\tau \in \mathbb R$, and the minimum value is:
    \begin{equation}
        \min_{\tau \in \mathbb R}\phi_\epsilon(x,\tau) = \min_{\tau \geq 1}\phi_\epsilon(x,\tau) = \log\left(\frac{\abs{x}}{\epsilon}+1\right).
    \end{equation}
\end{lemma}
\begin{proof}
    Let $z \coloneqq (\abs{x}/\epsilon+1)^2$. Since $x \in \mathbb{R}$ and $\epsilon > 0$, notice that $z \geq 1$.
    Considering $\phi_\epsilon(x, \tau)$ as a function of $\tau$ for $\tau \geq 1$, its derivative is
    \begin{equation}
        \frac{\partial \phi_\epsilon}{\partial \tau}(x, \tau) = -\frac{z}{2\tau^2} + \frac{1}{2\tau} = \frac{\tau - z}{2\tau^2}.
    \end{equation}
    Thus, $\phi_\epsilon(x, \tau)$ is strictly decreasing for $1 \le \tau < z$ and strictly increasing for $\tau > z$, confirming that $\tau^* = z$ is the unique minimizer.
    Substituting $\tau^* = z$ yields the minimum value:
    \begin{equation}
        \phi_\epsilon(x, z) = \frac{z}{2z} + \frac12\log z - \frac12 = \log\qty(\frac{\abs{x}}{\epsilon}+1).
    \end{equation}
    This completes the proof.
\end{proof}

\subsubsection{Definition of LogLOP-\ltwolone}
Substituting this $\phi_\epsilon$ into the general framework \eqref{eq:general-lop-framework} yields the LogLOP-\ltwolone penalty:
\begin{align}
    \Psi_{\alpha,\epsilon}(\bm x) \coloneqq \min_{\substack{\bm\sigma\in\R^{N} \\
            \norm{\bm{D\sigma}}_1 \leq \alpha}} \sum_{n=1}^N \phi_{\epsilon}(x_n,\sigma_n).
    \label{eq:relop-l2l1}
\end{align}
This penalty can discover the block partitions such as LOP-\ltwolone while evaluating the magnitude of these blocks using a logarithmic scale, directly solving the underestimation issue.

\subsubsection{Asymptotic Behavior}
The partition parameter $\alpha \geq 0$ seamlessly interpolates between the element-wise log-sum sparsity and a logarithmic block-sparse penalty:
\begin{theorem}[Asymptotic Behavior of LogLOP-\ltwolone]
    \zlabel{thm:var-rwl1}
    (i) As $\alpha\to\infty$ (finest partitions), $\Psi_{\alpha,\epsilon}$ recovers the log-sum penalty:
    \begin{equation}
        \lim_{\alpha\to\infty} \Psi_{\alpha,\epsilon}(\bm x) = \Omega_{\epsilon}(\bm x) = \sum_{n=1}^N \log\left(\frac{\abs{x_n}}{\epsilon}+1\right).
    \end{equation}
    (ii) As $\alpha\to0$ (coarsest partitions), it converges to a logarithmic $\ell_2$ penalty as shown in \zcref{fig:relop_alpha_0}:
    \begin{equation}
        \lim_{\alpha\to0} \Psi_{\alpha,\epsilon}(\bm x) = \Psi_{0,\epsilon}(\bm x)  = N\log\sqrt{\frac1N{\sum_{n=1}^N \left(\frac{\abs{x_n}}{\epsilon}+1\right)^2}}.
    \end{equation}
\end{theorem}
\begin{proof}
    See \zcref{app:var-rwl1} analogous to~\cite[App. C]{9729560}.
\end{proof}

\subsubsection{Validity for Block Partitioning}
A critical requirement for any LOP-based penalty is the ability to distinguish nonzero blocks from zero blocks as discussed in~\cite[Remark 1]{9729560}.
We define the penalty on a specific block $\mathcal B$ as:
\begin{equation}
    \small
    \Theta_{\epsilon}(\bm x_{\mathcal B}) \coloneqq \min_{\tau \geq 1} \sum_{n \in \mathcal B} \phi_{\epsilon}(x_n, \tau) = \abs{\mathcal B} \log \sqrt{\frac{1}{\abs{\mathcal B}} \sum_{n \in \mathcal B} \left(\frac{|x_n|}{\epsilon}+1\right)^2}.
\end{equation}
The LogLOP-\ltwolone penalty satisfies the following property:

\begin{proposition}[Block Decomposition Property]
    \zlabel{prop:monotonicity-of-f}
    For any block $\mathcal B$ decomposable into a nonzero subblock $\mathcal B'$ ($\bm x_{\mathcal B'} \neq \bm 0$) and a zero-valued subblock $\mathcal B''$ ($\bm x_{\mathcal B''} = \bm 0$), the penalty strictly decreases upon separation:
    \begin{equation}
        \Theta_{\epsilon}(\bm x_{\mathcal B}) > \Theta_{\epsilon}(\bm x_{\mathcal B^{'}}) + \Theta_{\epsilon}(\bm x_{\mathcal B^{''}}) = \Theta_{\epsilon}(\bm x_{\mathcal B^{'}}),
        \label{eq:block_decomposition_prop}
    \end{equation}
    where $\Theta_{\epsilon}$ represents the penalty value on a block.
\end{proposition}
\begin{proof}
    We analyze the penalty difference $\Delta$ before and after separating the zero block. Let $z \coloneqq \sum_{n \in \mathcal B'} (\abs{x_n}/\epsilon+1)^2$. Since $(\abs{x_n}/\epsilon+1)^2 \geq 1$ holds for all $n$ and strictly exceeds 1 if $x_n \ne 0$, $\bm x_{\mathcal B'} \neq \bm 0$ guarantees $z > \abs{\mathcal B'}$. We show that $\Delta(z)$ is strictly increasing for $z > \abs{\mathcal B'}$ and $\Delta(\abs{\mathcal B'}) = 0$. This implies $\Delta(z) > 0$ for any $z > \abs{\mathcal B'}$, proving that separation always reduces the penalty. The proof is provided in \zcref{app:monotonicity-of-f}.
\end{proof}
This property guarantees that the LogLOP-\ltwolone penalty actively drives the partitions to separate nonzero signal components from zero regions, ensuring accurate support recovery.

\subsection{Adaptively Weighted LOP-\ltwolone (AdaLOP-\ltwolone)}
\subsubsection{Variational Formulation}
We realize the AdaLOP-\ltwolone penalty by choosing an nonconvex variational function $\phi_{\gamma, w}$. Inspired by the adaptive weighting mechanism of EMCP~\cite{pokalaIterativelyReweightedMinimaxConcave2022,zhangTensorRecoveryBased2023}, we define this function explicitly as:
\begin{equation}
    \phi_{\gamma, w}(x, \tau) \coloneqq \begin{cases}
        w \phi(x, \tau) - \frac{\phi(x, \tau)^2}{2\gamma}, & \text{if } \phi(x, \tau) \leq \gamma w; \\
        \frac{\gamma w^2}{2},                              & \text{if } \phi(x, \tau) > \gamma w,
    \end{cases}
    \label{eq:var-lop-l2l1-def}
\end{equation}
where $\phi(x, \tau)$ is the variational function \zcref{eq:var-l2norm}.
This function provides a bounded penalty and can be equivalently expressed through a variational form involving an adaptive weight $\omega$:
\begin{equation}
    \phi_{\gamma, w}(x, \tau) = \min_{\omega \in \mathbb R_{\geq 0}} \left\{ \omega \phi(x, \tau) + \frac{\gamma}{2} (\omega - w)^2 \right\}.
    \label{eq:eq-var-lop-l2l1}
\end{equation}
The minimum is achieved at $\omega = \max(0, w - \gamma^{-1}\phi(x, \tau))$.
Here, we adopt the rule $0 \cdot \infty = 0$. This ensures that even when $\phi(x, \tau) = \infty$, the minimization is well-defined (attained at $\omega=0$) and consistent with the explicit definition \eqref{eq:var-lop-l2l1-def}, preserving the lower semicontinuity required for the ADMM optimization described in \zcref{sec:optimization}.
\begin{lemma}[Minimum Value of $\phi_{\gamma, w}(x, \tau)$]
    \zlabel{lem:var-ada-lop-l2l1}
    For any fixed $x \in \mathbb R$, $\gamma > 0$, and $w > 0$, the minimum value of $\phi_{\gamma, w}(x, \tau)$ with respect to $\tau \in \mathbb R$ recovers the MCP $\Omega_{\gamma, w}(x)$:
    \begin{equation}
        \min_{\tau \in \mathbb R} \phi_{\gamma, w}(x, \tau) = \min_{\tau \in \mathbb R_{\geq 0}} \phi_{\gamma, w}(x, \tau) = \Omega_{\gamma, w}(x).
    \end{equation}
    The minimum is attained at $\tau = \abs{x}$ i.e. $\phi(x, \tau) = \abs{x}$.
\end{lemma}
\begin{proof}
    See \zcref{app:proof_adalopl2l1_upper_bounds_mcp}.
\end{proof}

\subsubsection{Definition of AdaLOP-\ltwolone}
Substituting $\phi_{\gamma, w}$ into the general framework \eqref{eq:general-lop-framework} yields the AdaLOP-\ltwolone penalty:
\begin{align}
     & \Psi_{\alpha, \gamma, \bm w} (\bm x) \coloneqq \min_{\substack{\bm\sigma\in\R^{N} \\ \norm{\bm{D\sigma}}_1 \leq \alpha}} \sum_{n=1}^N \phi_{\gamma, w_n}(x_n, \sigma_n) \nonumber \\
     & = \min_{\substack{\bm\sigma\in\R^{N}                                              \\ \norm{\bm{D\sigma}}_1 \leq \alpha}} \min_{\bm\omega \in \mathbb R_{\geq 0}^N} \left\{ \sum_{n=1}^N \omega_n \phi(x_n, \sigma_n) + \frac{\gamma}{2} \norm{\bm\omega - \bm w}_2^2 \right\}.
    \label{eq:ada-lop-l2l1}
\end{align}
This formulation performs triple optimization: signal estimation ($\bm x$), partition discovery ($\bm \sigma$), and weighting ($\bm \omega$). Note that although $\bm w$ appears as a fixed parameter vector in this definition, it is iteratively updated in the proposed algorithm to adaptively control the regularization strength (see \zcref{sec:optimization}).

\subsubsection{Asymptotic Behavior}
The parameter $\alpha \geq 0$ controls the granularity of the block partitions such as LOP-\ltwolone.
\begin{theorem}[Asymptotic Behavior of AdaLOP-\ltwolone]
    (i) As $\alpha \to \infty$, $\Psi_{\alpha, \gamma, \bm w}(\bm x)$ converges to the vector EMCP:
    \begin{equation}
        \Omega_{\gamma, \bm w}(\bm x) = \min_{\bm \omega \in \mathbb R_{\geq 0}^N} \left\{\norm{\bm x}_{\bm \omega,1} + \frac{\gamma}{2} \norm{\bm \omega - \bm w}_2^2\right\}.
    \end{equation}
    (ii) As $\alpha \to 0$, it converges to a globally adaptive weighted $\ell_2$ penalty as shown in \zcref{fig:adalop_alpha_0}:
    \begin{equation}
        \Psi_{0, \gamma, \bm w}(\bm x) = \min_{\bm\omega \in \mathbb R_{\geq 0}^N} \left\{ \norm{\bm\omega}_1^{\frac12}\norm{\bm x}_{\bm \omega, 2} + \frac{\gamma}{2} \norm{\bm\omega - \bm w}_2^2 \right\}.
        \label{eq:ada-l2}
    \end{equation}
    where $\norm{\bm x}_{\bm \omega, 2} = \sqrt{\sum_{n=1}^N \omega_n \abs{x_n}^2}$ is the weighted $\ell_2$ penalty.
    \zlabel{thm:ada-lop-l2l1}
\end{theorem}
\begin{proof}
    See \zcref{app:ada-lop-l2l1} analogous to~\cite[App. C]{9729560}.
\end{proof}
This theorem confirms that the AdaLOP-\ltwolone penalty is a rigorous generalization of the MCP to the LOP-type penalty.

\begin{remark}[Parameter Selection for $\gamma$]
    The parameter $\gamma$ not only controls the penalty curvature but also acts as the inverse step size for the weight update (see \zcref{eq:omega-update}).
    Our experiments show that a smaller $\gamma$ (larger step size) is preferred when the sparsity parameter $w$ is large. This accelerates the weight decay on identified nonzero blocks, enhancing accurate recovery.
\end{remark}
\section{Optimization Algorithm}
\zlabel{sec:optimization}

\subsection{Implementation of Proposed Methods}
We solve the optimization problems for the LogLOP-\ltwolone and AdaLOP-\ltwolone penalities using the ADMM~\cite{boydDistributedOptimizationStatistical2010}. This framework effectively decouples the nonconvex regularizers from the data fidelity term, enabling efficient iterative updates.

For LogLOP-\ltwolone, the problem is formulated as:
\begin{equation}
    \mathop{\mathrm{minimize}}_{\bm x \in \R^N} f(\bm{Ax}) + \lambda\Psi_{\alpha,\epsilon}(\bm{Lx}).
    \label{eq:min-relop-l2l1}
\end{equation}
Similarly, for AdaLOP-\ltwolone:
\begin{equation}
    \mathop{\mathrm{minimize}}_{\bm x \in \R^N} f(\bm{Ax}) + \Psi_{\alpha, \gamma, \bm w}(\bm{Lx}).
    \label{eq:min-adalop-l2l1}
\end{equation}
Note that in AdaLOP-\ltwolone, following the approach in~\cite{pokalaIterativelyReweightedMinimaxConcave2022}, the weight vector $\bm w$ is also optimized iteratively to actively mitigate the estimation bias.
In the following, we assume that the function $f$ is proper, lsc, coercive and convex for simplicity.
\zcref{alg:relopl2l1,alg:adalopl2l1} present the detailed algorithms. We describe the specific update rules for each variable below. \zcref{app:relopl2l1-alg,app:adalopl2l1-alg} provide the full derivations.
Here, $\mu_1, \dots, \mu_4$ denote the penalty parameters for the augmented Lagrangian constraints, and $\rho \geq 1$ represents their update rate.
For LogLOP-\ltwolone, we suggest setting $\mu_4 = \lambda/54 + \delta$ with a small constant $\delta > 0$ to guarantee the strict convexity of the $\xi$-update (see \zcref{app:relopl2l1-alg}).

\textbf{1) $\bm x$-update (Linear System):}
Both methods share the same $\bm x$-update, minimizing quadratic penalty terms. Minimizing with respect to $\bm x$ yields the following linear system:
\begin{equation}
    \small
    \qty(\mu_1\bm A^\top\bm A + \mu_2\bm L^\top\bm L)\bm x = \mu_1\bm A^\top\qty(\bm u + \frac{\bm r_1}{\mu_1}) + \mu_2\bm L^\top\qty(\bm v + \frac{\bm r_2}{\mu_2}).
    \label{eq:x-update}
\end{equation}
This system is positive semidefinite and can be solved efficiently using conjugate gradient (CG) or FFT-based solvers\footnote{It is possible to use FFT-based solvers if the matrices $\bm A$ and $\bm L$ are diagonalizable by FFT, e.g., circulant matrices due to periodic boundary conditions.}.

\textbf{2) $\bm u$-update (Proximal Operator of $f$):}
This step applies the proximal operator of the data fidelity term:
\begin{equation}
    \bm u \leftarrow \prox_{\mu_1^{-1}f}\qty(\bm{Ax} - \mu_1^{-1}\bm r_1).
    \label{eq:u-update}
\end{equation}
For squared error $f(\bm u) = \frac{1}{2}\norm{\bm y - \bm u}_2^2$, $\prox_{\beta f}(\bm u) = {(\beta \bm y+\bm u)}/{(\beta+1)}$. For absolute error $f(\bm u) = \norm{\bm y-\bm u}_1$, $\prox_{\beta f}(\bm u) = \left(u_j + \beta(y_j - u_j)/\max\{|y_j - u_j|, \beta\}\right)_{j=1}^{J}$.

\textbf{3) $\bm \sigma$-update (Linear System):}
This update links the auxiliary variables $\bm \eta$ and $\bm \xi$ and minimizes quadratic penalty terms:
\begin{equation}
    \qty(\mu_3\bm D^\top\bm D + \mu_4\bm I)\bm\sigma = \mu_3\bm D^\top\qty(\bm\eta + \mu_3^{-1}\bm r_3) + \mu_4\qty(\bm\xi + \mu_4^{-1}\bm r_4).
    \label{eq:sigma-update}
\end{equation}

\textbf{4) $\bm \eta$-update (Projection onto $\ell_1$-ball):}
The variable $\bm \eta$ handles the constraint on the total variation of $\bm \sigma$. The update is the projection onto the $\ell_1$-ball $B_1^{\alpha} \coloneqq \{ \bm z \mid \norm{\bm z}_1 \leq \alpha \}$:
\begin{equation}
    \bm\eta \leftarrow P_{B_1^{\alpha}}\qty(\bm{D\sigma} - \mu_3^{-1}\bm r_3).
    \label{eq:eta-update}
\end{equation}
The projection $P_{B_1^{\alpha}}(\bm z)$ is computed as $\bm z$ if $\norm{\bm z}_1 \le \alpha$, and otherwise as $(\max\{\abs{z_k} - \theta, 0\} \cdot \mathrm{sign}(z_k))_{k=1}^K$, where $\theta$ is a threshold found in expected linear time~\cite{condatFastProjectionSimplex2016a}.

\textbf{5) $\bm v$ and $\bm \xi$ Updates:}
The updates for $\bm v$ and $\bm \xi$ differ between the two methods.

\textit{LogLOP-\ltwolone:}
The $\bm v$-update is the proximal operator of the Elastic Net~\cite{chierchiaProximityOperatorRepository2016}:
\begin{equation}
    \bm v \leftarrow \prox_{\lambda_1\norm{\cdot}_1 + \lambda_2\norm{\cdot}_2^2}\qty(\bm{Lx} - \mu_2^{-1}\bm r_2),
    \label{eq:v-update-relop}
\end{equation}
where $\lambda_1 = \lambda/(\mu_2 \xi_n \epsilon)$, $\lambda_2 = \lambda/(2\mu_2 \xi_n \epsilon^2)$ and
\begin{equation}
    \prox_{\lambda_1\norm{\cdot}_1 + \lambda_2\norm{\cdot}_2^2}(\bm v) = \frac{1}{1+2\lambda_2} \sign(\bm v) \odot \max\qty{0, \abs{\bm v} - \lambda_1}.
\end{equation}
The $\bm \xi$-minimization problem is strictly convex and admits a unique solution under the sufficient condition $\mu_4 > \lambda/54$ (see \zcref{app:relopl2l1-alg}). For each component, we solve the problem via the bisection method to find the optimal $\xi_n \geq 1$:
\begin{equation}
    \small
    \xi_n \leftarrow \argmin_{\xi \geq 1} \frac\lambda{2\xi}\qty(\frac{\abs{v_n}}{\epsilon}+1)^2 + \frac\lambda2\log\xi + \frac{\mu_4}{2}\qty(\xi - \sigma_n + \frac{r_{4,n}}{\mu_4})^2.
    \label{eq:xi-update-relop}
\end{equation}
The appropriate initial points for the bisection interval are given by $\xi_n = \max\qty{1, \min\qty{a, b}}$ and $\xi_n = \max\qty{a, b}$, where $a = (\abs{v_n}/\epsilon+1)^2$ and $b = \sigma_n - r_{4,n}/\mu_4$.
The details of strict convexity and the interval are given in \zcref{app:relopl2l1-alg}.

\textit{AdaLOP-\ltwolone:}
The variables $\bm v$ and $\bm \xi$ are updated jointly via the proximal operator of the weighted variational function:
\begin{equation}
    (\bm v, \bm\xi) \leftarrow \prox_{(\mu_2^{-1}\bm \omega) \odot \phi}\qty(\bm{Lx} - \mu_2^{-1}\bm r_2, \bm\sigma - \mu_4^{-1}\bm r_4).
    \label{eq:v-xi-update-adalop}
\end{equation}
This operator admits a closed-form solution~\cite{combettesPerspectiveMaximumLikelihoodtype2020}:
\begin{align*}
     & \prox_{\omega\phi}(\bm v, \bm\xi) \\ &= \begin{cases}
        (0,0),                                                   & \mif\ 2\omega\xi + \abs{v}^2 \leq \omega^2; \\
        (0, \xi - \omega/2),                                     & \mif\ v = 0\ \mand\ 2\xi > \omega;          \\
        \qty(v - s\omega\sign(v), \xi + \frac{s^2-1}{2}\omega ), & \text{otherwise},
    \end{cases}
\end{align*}
where $s > 0$ is the unique positive root of $s^3 + (2\xi/\omega+1)s - 2\abs{v}/\omega = 0$, solvable via Cardano's formula~\cite{9729560,kurodaConvexnonconvexFrameworkEnhancing2025, bauschkeRealRootsReal2023}.
Additionally, the weights $\bm \omega$ are updated as for any $n \in [N]$:
\begin{equation}
    \omega_n \leftarrow \max\qty{0, w_n - \gamma^{-1}{\phi(v_n, \xi_n)}}.
    \label{eq:omega-update}
\end{equation}

\textbf{6) Dual Variable Updates:}
The Lagrange multipliers are updated via standard ascent:
\begin{equation}
    \left\{
    \begin{aligned}
        \bm r_1 & \leftarrow \bm r_1 + \mu_1(\bm u - \bm{Ax}),        \\
        \bm r_2 & \leftarrow \bm r_2 + \mu_2(\bm v - \bm{Lx}),        \\
        \bm r_3 & \leftarrow \bm r_3 + \mu_3(\bm\eta - \bm{D\sigma}), \\
        \bm r_4 & \leftarrow \bm r_4 + \mu_4(\bm\xi - \bm\sigma).
    \end{aligned}
    \right.
    \label{eq:dual-update}
\end{equation}

\begin{algorithm}[tb]
    \caption{ADMM algorithm for LogLOP-$\ell_{2}$/$\ell_{1}$ \zcref{eq:min-relop-l2l1}}
    \zlabel{alg:relopl2l1}
    \begin{algorithmic}[1]
        \Require $\bm x^0, \bm \sigma^0, \bm u^0, \bm v^0, \bm \eta^0, \bm \xi^0, \bm r_1^0, \bm r_2^0, \bm r_3^0, \bm r_4^0, \epsilon > 0, \lambda \geq 0, \alpha \geq 0, \mu_1 > 0, \mu_2 > 0, \mu_3 > 0, \mu_4 > 0, \rho \geq 1$
        \For{$k \to0, 1, \dots, k_{\mathrm{max}}$}
        \State Updating $\bm x^{k+1}$ via \zcref{eq:x-update}
        \State $\bm u^{k+1} \leftarrow \prox_{\mu_1^{-1}f}\qty(\bm{Ax}^{k+1} - \mu_1^{-1}\bm r_1^k)$ \zcref{eq:u-update}
        \State $\bm v^{k+1} \leftarrow \prox_{\lambda_1\norm{\cdot}_1 + \lambda_2\norm{\cdot}_2^2}\qty(\bm{Lx}^{k+1} - \mu_2^{-1}\bm r_2^k)$ \zcref{eq:v-update-relop}
        \State Updating $\bm\sigma^{k+1}$ via \zcref{eq:sigma-update}
        \State $\bm\eta^{k+1} \leftarrow P_{B_1^{\alpha}}\qty(\bm{D\sigma}^{k+1} - \mu_3^{-1}\bm r_3^k)$ \zcref{eq:eta-update}
        \State Updating $\bm \xi^{k+1}$ via \zcref{eq:xi-update-relop}
        \State Updating $(\bm r_1^{k+1}, \bm r_2^{k+1}, \bm r_3^{k+1}, \bm r_4^{k+1})$ via \zcref{eq:dual-update}
        \State $(\mu_1^{k+1}, \mu_2^{k+1}, \mu_3^{k+1}, \mu_4^{k+1}) \leftarrow \rho \cdot (\mu_1^k, \mu_2^k, \mu_3^k, \mu_4^k)$
        \EndFor
        \Ensure $\bm x^{k_{\mathrm{max}}+1} \in \R^N, \bm \sigma^{k_{\mathrm{max}}+1} \in \R^K$
    \end{algorithmic}
\end{algorithm}

\begin{algorithm}[tb]
    \caption{ADMM algorithm for AdaLOP-$\ell_{2}$/$\ell_{1}$ \zcref{eq:min-adalop-l2l1}}
    \zlabel{alg:adalopl2l1}
    \begin{algorithmic}[1]
        \Require $\bm x^0, \bm \sigma^0, \bm u^0, \bm v^0, \bm \eta^0, \bm \xi^0, \bm \omega^0, \bm w^0, \bm r_1^0, \bm r_2^0, \bm r_3^0, \bm r_4^0, \gamma > 1, \alpha \geq 0, \mu_1 > 0, \mu_2 > 0, \mu_3 > 0, \mu_4 = \mu_2, \rho \geq 1$
        \For{$k \to0, 1, \dots, k_{\mathrm{max}}$}
        \State $\omega_n^{k+1} \leftarrow \max\qty{0, w_n^k - \gamma^{-1}{\phi(v_n^k, \xi_n^k)}}$ \zcref{eq:omega-update}
        \State $\bm w^{k+1} \leftarrow \bm \omega^{k+1}$
        \State Updating $\bm x^{k+1}$ via \zcref{eq:x-update}
        \State $\bm u^{k+1} \leftarrow \prox_{\mu_1^{-1}f}\qty(\bm{Ax}^{k+1} - \mu_1^{-1}\bm r_1^k)$ \zcref{eq:u-update}
        \State Updating $\bm\sigma^{k+1}$ via \zcref{eq:sigma-update}
        \State $\bm\eta^{k+1} \leftarrow P_{B_1^{\alpha}}\qty(\bm{D\sigma}^{k+1} - \mu_3^{-1}\bm r_3^k)$ \zcref{eq:eta-update}
        \State Updating $(\bm v^{k+1}, \bm\xi^{k+1})$ via \zcref{eq:v-xi-update-adalop}
        \State Updating $(\bm r_1^{k+1}, \bm r_2^{k+1}, \bm r_3^{k+1}, \bm r_4^{k+1})$ via \zcref{eq:dual-update}
        \State $(\mu_1^{k+1}, \mu_2^{k+1}, \mu_3^{k+1}, \mu_4^{k+1}) \leftarrow \rho \cdot (\mu_1^k, \mu_2^k, \mu_3^k, \mu_4^k)$
        \EndFor
        \Ensure $\bm x^{k_{\mathrm{max}}+1} \in \R^N, \bm \sigma^{k_{\mathrm{max}}+1} \in \R^K$
    \end{algorithmic}
\end{algorithm}

\begin{remark}[Convergence of \zcref{alg:relopl2l1, alg:adalopl2l1}]
    Because our algorithms employ alternating minimization, the augmented Lagrangian value (shown \zcref{eq:relopl2l1-augmented-lagrange,eq:adalopl2l1-augmented-lagrange} in \zcref{app:relopl2l1-alg, app:adalopl2l1-alg}) decreases monotonically. It is also bounded below by zero and hence the value converges to a nonnegative value.
    However, proving sequence convergence is challenging. Although our objective function satisfies the Kurdyka-\L ojasiewicz (KL) property (see \zcref{app:kl-phi,app:kl-phi-epsilon,app:kl-psi,app:kl-phi-gamma-w}), existing analyses~\cite{wangGlobalConvergenceADMM2019, botProximalAlternatingDirection2020} relying on this property assume surjective linear constraints as mentioned in a recent survey~\cite{hanSurveyRecentDevelopments2022}. In our formulation, the linear constraints are given by
    \begin{equation}
        \begin{pmatrix} -\bm I & \bm O & \bm O & \bm O \\ \bm O & -\bm I & \bm O & \bm O \\ \bm O & \bm O & -\bm I & \bm O \\ \bm O & \bm O & \bm O & -\bm I \end{pmatrix}
        \begin{pmatrix} \bm u \\ \bm v \\ \bm \eta \\ \bm \xi \end{pmatrix}
        +
        \underbrace{\begin{pmatrix} \bm A & \bm O \\ \bm L & \bm O \\ \bm O & \bm D \\ \bm O & \bm I \end{pmatrix}}_{\eqqcolon\bm C}
        \begin{pmatrix} \bm x \\ \bm \sigma \end{pmatrix}
        = \bm 0,
    \end{equation}
    where $\bm C$ is generally not surjective, violating the assumption.
    Alternative penalty methods~\cite{attouchProximalAlternatingMinimization2010} theoretically avoid this issue but lead to ill-conditioned subproblems.
    Despite these theoretical gaps, \zcref{alg:relopl2l1, alg:adalopl2l1} exhibit stable empirical convergence, supported by the convexity of the subproblems.
    \zlabel{rem:algs-convergence}
\end{remark}

\subsection{Computational Complexity Analysis}
We analyze the computational complexity of the proposed algorithms, LogLOP-\ltwolone (\zcref{alg:relopl2l1}) and AdaLOP-\ltwolone (\zcref{alg:adalopl2l1}), focusing on their per-iteration costs. The complexity is influenced by the dimensions of the target signal $\bm x$, the observation $\bm y$, and the transformed signal $\bm{Lx}$, as well as the costs associated with applying the operators $\bm A$ and $\bm L$.

Let $N$ be the target signal dimension, $J$ be the observation dimension, $K$ be the transformed signal dimension, $C_A$ and $C_L$ be the costs of applying operators $\bm A, \bm A^\top$ and $\bm L, \bm L^\top$ respectively, $N_{\mathrm{cg}}$ be the number of conjugate gradient iterations (determined by the condition number and tolerance~\cite{shewchukIntroductionConjugateGradient1994}), and $N_{\mathrm{bisect}}$ be the number of bisection iterations (determined by the initial interval and tolerance).

For LogLOP-\ltwolone, the per-iteration complexity is:
\begin{equation}
    \mathcal{O}(N_{\mathrm{cg}}(C_A + C_L) + N_{\mathrm{bisect}} \cdot K)
\end{equation}

For AdaLOP-\ltwolone, the per-iteration complexity is:
\begin{equation}
    \mathcal{O}(N_{\mathrm{cg}}(C_A + C_L) + K)
\end{equation}

When $\bm A$ and $\bm L$ are sparse operators such as difference operators, $C_A, C_L \approx \mathcal{O}(N)$.
If they are diagonalizable by the Fourier transform, the linear system can be solved directly in $\mathcal{O}(N \log N)$ operations, instead of using conjugate gradient.
\section{Numerical Experiments}
\zlabel{sec:experiments}
We evaluate the proposed LogLOP-\ltwolone and AdaLOP-\ltwolone penalties through numerical experiments. We demonstrate their effectiveness in bias mitigation and block partition recovery using both synthetic data for controlled analysis and practical applications including Angular Power Spectrum (APS) estimation and denoising of nanopore currents.

\subsection{Synthetic Data Example in Compressive Sensing}

\begin{figure}[tb]
    \centering
    \includegraphics[width=0.9\linewidth]{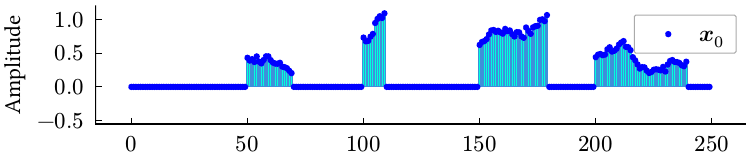}
    \vspace{0.5em}
    \includegraphics[width=0.9\linewidth]{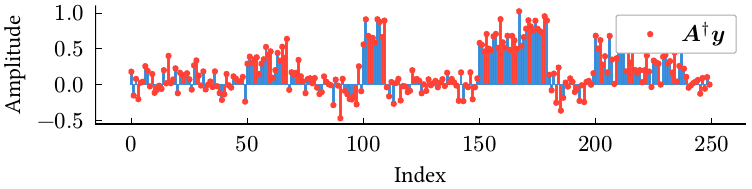}
    \caption{Experimental setup for compressive sensing. (Top) True block-sparse signal $\bm x_0 \in \R^{250}$ with four nonzero blocks. (Bottom) Initial pseudo-inverse recovery $\bm A^\dagger \bm y$ (SNR: 7.6 dB).}
    \zlabel{fig:cs_setup_signal}
\end{figure}

We first evaluate the proposed methods on a synthetic compressive sensing task, focusing on the bias mitigation and the recovery of the true support.

\subsubsection{Experimental Setup}
We consider the linear model $\bm y = \bm A \bm x + \bm \epsilon$, recovering a block-sparse signal $\bm x \in \mathbb{R}^N$ from noisy measurements $\bm y \in \mathbb{R}^J$.
We set $N=250$ and $J=200$. The sensing matrix $\bm A$ has i.i.d. entries from $\mathcal{N}(0,1)$. The ground truth $\bm x_0$ (\zcref{fig:cs_setup_signal}) contains four nonzero blocks with intra-block variations. Gaussian noise $\bm \epsilon \sim \mathcal{N}(0, \sigma^2 \bm I)$ is added to achieve an input SNR of 40 dB, where $\text{SNR} \coloneqq 20 \log_{10} (\norm{\bm x_0}_2 / \norm{\bm x - \bm {x}_0}_2)$.

\subsubsection{Performance Evaluation}
We compare LogLOP-\ltwolone and AdaLOP-\ltwolone against five baselines: $\ell_1$ norm, convex LOP-\ltwolone~\cite{9729560}, nonconvex GME-LOP-\ltwolone~\cite{kurodaConvexnonconvexFrameworkEnhancing2025}, PC-SBL~\cite{fangPatterncoupledSparseBayesian2015}, and TV-SBL~\cite{santBlockSparseSignalRecovery2022}.
For both PC-SBL and TV-SBL, we used the $\lambda$ estimated by the PC-SBL solver.
All methods solve the regularized problem $\min_{\bm{x}} \frac{1}{2}\norm{\bm{y} - \bm{Ax}}_2^2 + \mathcal{R}(\bm{x})$, where $\mathcal{R}(\bm{x}) = \lambda \Psi(\bm{x})$ for $\ell_1$, LOP, GME-LOP, and LogLOP-\ltwolone, and $\mathcal{R}(\bm{x}) = \Psi_{\alpha, \gamma, \bm w}(\bm{x})$ for AdaLOP-\ltwolone with initial weights $\bm w = \lambda \bm 1$.
For LOP-based methods, the block partition parameter is fixed at $\alpha=40.0$ which is nearly optimally chosen for LOP-\ltwolone.
For all methods, we used random values drawn from an i.i.d. standard normal distribution as the initial points to mitigate the initialization bias.

\textbf{Results:}
\begin{figure}[tb]
    \centering
    {\includegraphics[width=0.9\linewidth]{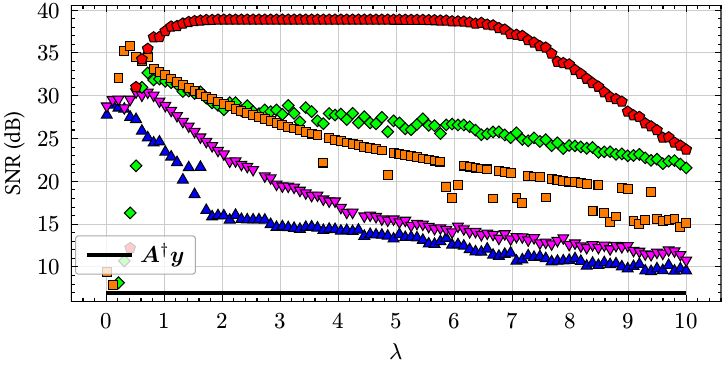}}
    \\
    {\includegraphics[width=0.9\linewidth]{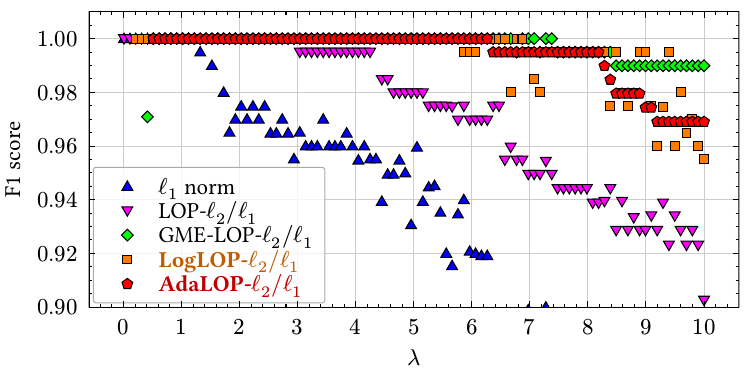}}
    \caption{Quantitative evaluation results. (Top) SNR vs. $\lambda$. (Bottom) F1 score vs. $\lambda$. AdaLOP-\ltwolone achieves the highest SNR, maintaining robustness across a wide $\lambda$ range. The proposed methods maintain near-perfect support recovery (F1 $\approx$ 1.0) significantly better than convex baselines.}
    \zlabel{fig:quantitative_results}
\end{figure}
\begin{figure}[tb]
    \centering
    \includegraphics[width=0.9\linewidth]{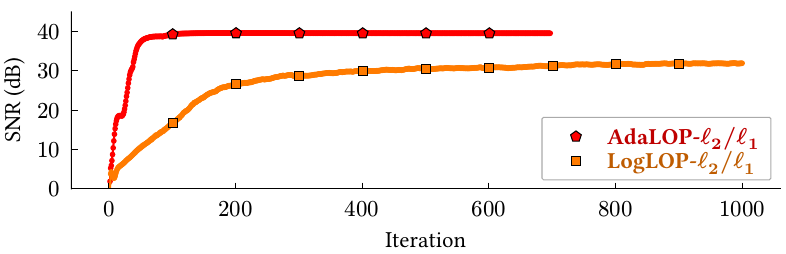}
    \caption{Convergence of SNR with fixed $\lambda=1.0$, $\alpha=40.0$.}
    \zlabel{fig:compare_convergence}
\end{figure}
\zcref{fig:quantitative_results} shows the quantitative results of a single trial between LOP-based methods. AdaLOP-\ltwolone consistently achieves the highest SNR and maintains near-perfect F1 scores for support recovery across a wide range of $\lambda$. This highlights the robustness of the proposed nonconvex formulations compared to the rapid degradation of convex methods. Furthermore, \zcref{fig:compare_convergence} confirms that the proposed methods stably converge to superior solutions from random initial points (avg. over 100 independent trials, $\lambda=1.0$).

Qualitatively, \zcref{fig:cs_qualitative_comparison_lambda1} demonstrates the bias mitigation capability. Although convex methods suffer from amplitude shrinkage, AdaLOP-\ltwolone recovers the signal amplitudes with high accuracy (38.91 dB), closely matching the ground truth.

\begin{figure}[tb]
    \centering
    \includegraphics[width=0.9\linewidth]{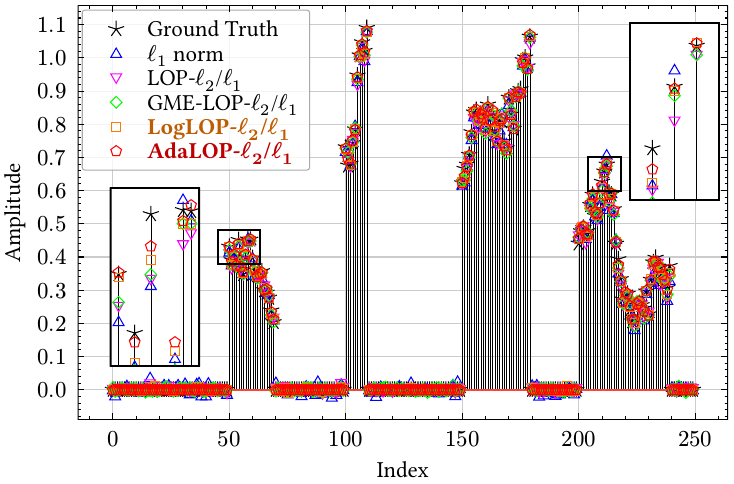}
    \caption{
        Qualitative comparison of recovered signals in a single trial.
        Convex methods ($\ell_1$, LOP-\ltwolone) exhibit amplitude shrinkage.
        In contrast, AdaLOP-\ltwolone achieves highly accurate recovery, effectively mitigating the bias.
    }
    \zlabel{fig:cs_qualitative_comparison_lambda1}
\end{figure}

\zcref{fig:robustness_analysis} further analyzes robustness, where the results are averaged over 500 independent trials. AdaLOP-\ltwolone consistently outperforms baselines across varying observation dimensions $J$ and noise levels $\sigma$, confirming its reliability in diverse conditions. For this evaluation, the parameter $\lambda$ was tuned to the best value for each method.

\begin{figure}[tb]
    \centering
    \includegraphics[width=0.9\linewidth]{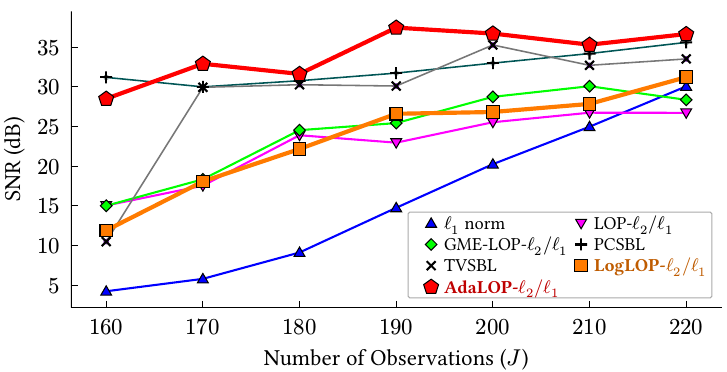}
    \vspace{1em}
    \includegraphics[width=0.9\linewidth]{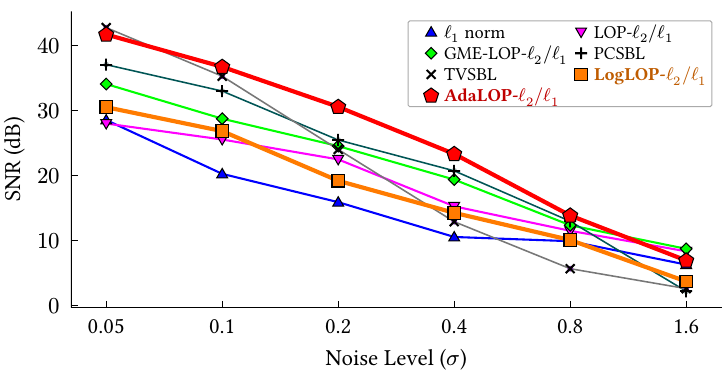}
    \caption{Robustness analysis (averaged over 500 independent trials). (Top) SNR versus number of observations $J$ with fixed $\sigma = 0.1$. (Bottom) SNR versus noise level $\sigma$ with fixed $J = 200$. AdaLOP-\ltwolone consistently outperforms baselines across varying conditions.}
    \zlabel{fig:robustness_analysis}
\end{figure}

\zcref{fig:sensitivity_plots} examines the sensitivity of AdaLOP-\ltwolone with respect to $\gamma$ (averaged over 20 trials). We observe that the optimal $\gamma$ decreases as $\lambda$ increases. Because $\gamma^{-1}$ acts as the learning rate for updating $\bm w$ in \zcref{eq:omega-update}, this implies that stronger regularization requires a larger step size for effective recovery.

\begin{figure}[tb]
    \centering
    \includegraphics[width=0.9\linewidth]{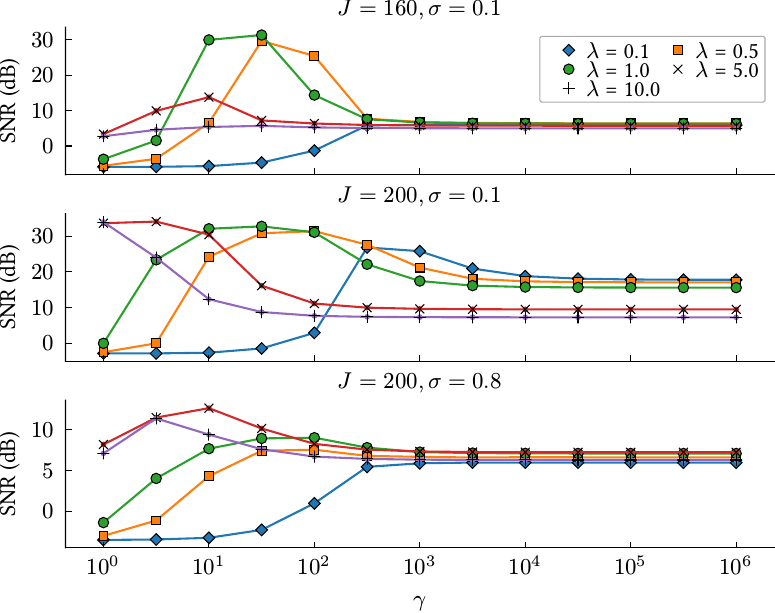}
    \caption{Parameter sensitivity of AdaLOP-\ltwolone with respect to $\gamma > 1$. As $\lambda$ increases, a smaller $\gamma$ is preferred to achieve higher SNR.}
    \zlabel{fig:sensitivity_plots}
\end{figure}

\subsection{Application to Angular Power Spectrum Estimation}
\zlabel{subsec:aps_estimation}

We apply the proposed methods to the problem of Angular Power Spectrum (APS) estimation in MIMO communication systems. This problem is often ill-posed due to the limited number of antennas compared to the angular resolution. Such a limitation is critical in practical applications~\cite{kurodaTheoreticalValidationLatent2025}, necessitating effective regularization.

\subsubsection{Problem Formulation}
Consider a Uniform Linear Array (ULA) with $M$ directive antennas following the 3GPP document~\cite{3gpp_tr38901_2017} and~\cite[Example 1]{kurodaTheoreticalValidationLatent2025}. The received signal covariance matrix $\bm R \in \mathbb{C}^{M \times M}$ is related to the APS $\bm x \in \mathbb{R}_{\geq 0}^N$ defined on an angular grid $\{\theta_n\}_{n=1}^N$ by: $\bm R = \sum_{n=1}^N x_n \bm a(\theta_n) \bm a(\theta_n)^\mathsf{H}$, where $\bm a(\theta)$ is the array response vector and $\mathbb{C}$ denotes the set of complex numbers. Given a sample covariance matrix $\hat{\bm R}$ computed from $T$ snapshots, the goal is to estimate $\bm x$.
We formulate this as a regularized optimization problem. Following~\cite{kurodaTheoreticalValidationLatent2025}, we incorporate a data-driven prior term based on past observations. The objective function is:
\begin{equation}
    \min_{\bm x \in \mathbb{R}_{\geq 0}^N} \frac{1}{2} \norm{\mathcal{A}(\bm x) - \hat{\bm r}}_2^2 + \frac{\mu}{2} \norm{\bm x - \bar{\bm x}}_{\bm P}^2 + \mathcal{R}(\bm x),
    \label{eq:aps_objective}
\end{equation}
where $\mathcal{A}(\bm x)$ is the linear operator from the APS to the vectorized covariance, $\hat{\bm r}$ is the observation vector extracted from $\hat{\bm R}$, and $\frac{\mu}{2} \norm{\bm x - \bar{\bm x}}_{\bm P}^2$ is the data-driven term with mean $\bar{\bm x}$ and precision matrix $\bm P$ derived from a dataset. Here, $\norm{\bm z}_{\bm P}^2 \coloneqq \bm z^\top \bm P \bm z$ denotes the squared Mahalanobis distance. The regularization term is defined as $\mathcal{R}(\bm{x}) = \lambda \Psi(\bm{x})$ for LOP, GME-LOP, and LogLOP-\ltwolone, and $\mathcal{R}(\bm{x}) = \Psi_{\alpha, \gamma, \bm w}(\bm{x})$ for AdaLOP-\ltwolone with initial weights $\bm w = \lambda \bm 1$.

\subsubsection{Experimental Setup}
The experimental setup follows~\cite{kurodaTheoreticalValidationLatent2025}. We simulate a scenario with $N=100$ angular grid points uniformly spaced in $[-\pi/2, \pi/2]$. The number of antennas $M$ varies from 4 to 32. The true APS consists of $Q \in \{1, 2\}$ Gaussian-shaped blocks with random positions and amplitudes. The sample covariance $\hat{\bm R}$ is estimated from $T=1000$ snapshots with an input SNR of 30 dB.
A dataset of $L=1000$ past APS samples is generated to compute the prior statistics $\bar{\bm x}$ and $\bm P$.

We compare the estimation accuracy of the proposed LogLOP and AdaLOP-\ltwolone against several baselines.
All block-sparse regularization approaches, including the SBL-based methods (PC-SBL~\cite{fangPatterncoupledSparseBayesian2015} and TV-SBL~\cite{santBlockSparseSignalRecovery2022}), incorporate the data-driven prior term $\frac{\mu}{2} \norm{\bm x - \bar{\bm x}}_{\bm P}^2$.
The competing methods include: the Hybrid method~\cite{cavalcanteHybridDataModel2020a} (which corresponds to setting $\lambda=0$ in \zcref{eq:aps_objective}), and the existing LOP and GME-LOP-\ltwolone.
We evaluate the Normalized Mean Squared Error (NMSE) of the estimated APS averaged over 100 independent trials. The hyperparameters are optimized via grid search for each method: $\mu \in \{10^i\}_{i=-2}^1$, $\lambda \in \{10^i\}_{i=-4}^0$, $\alpha \in \{10^i\}_{i=1}^4$, and $\gamma \in \{1,10,100\}$ on LOP, GME-LOP, LogLOP, AdaLOP-\ltwolone.
For both PC-SBL and TV-SBL, we used the $\lambda$ estimated internally by the PC-SBL solver, which is the standard parameter-free approach for SBL-based methods.

\subsubsection{Performance Evaluation}
\begin{figure}[tb]
    \centering
    \includegraphics[width=0.9\linewidth]{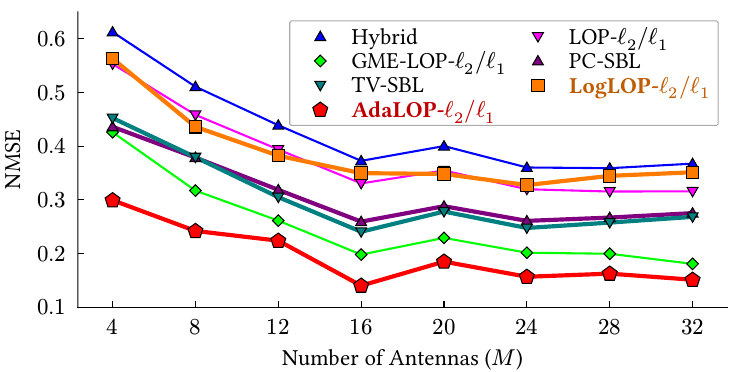}
    \caption{NMSE of APS estimation versus the number of antennas $M$ (averaged over 100 independent trials). AdaLOP-\ltwolone achieves the lowest NMSE.}
    \zlabel{fig:nmse_vs_antennas}
\end{figure}

As shown in \zcref{fig:nmse_vs_antennas}, AdaLOP-\ltwolone consistently achieves the lowest NMSE across all $M$, significantly outperforming all other methods including the state-of-the-art GME-LOP-\ltwolone and SBL-based approaches.
The performance gap is most pronounced when the number of antennas $M$ is small (e.g., $M=4, 8$), where the ill-posedness is severe. This indicates that the proposed methods effectively utilize the block sparsity prior to resolve angular ambiguity even with limited observations.

\subsection{Application to Denoising of Nanopore Currents}
\label{subsec:nanopore}

We apply the proposed methods to the denoising of nanopore currents, which is a critical step in nanopore-based DNA/RNA sequencing.
Nanopore currents are characterized by piecewise constant segments corresponding to different nucleotides, but they often exhibit transient slopes due to membrane capacitance and amplifier filtering~\cite{liangNoiseNanoporeSensors2020, wenGuideSignalProcessing2021}.
Although practical nanopore sensor noise is intricate and includes $1/f$ noise, this study models the noise environment by using a Mixed Poisson-Gaussian (MPG) distribution, which captures the signal-dependent shot noise and thermal noise~\cite{liangNoiseNanoporeSensors2020}.

To address these characteristics, we employ the Shifted I-divergence (SID) as the convex data fidelity term suitable for MPG noise~\cite{chakrabartiImageRestorationSignaldependent2012}.
We formulate the optimization problem as:
\begin{equation}
    \min_{\bm{x} \in \mathbb{R}_{\geq 0}^N} \sum_{n=1}^{N} \left\{ x_n + \nu^2 - y_n \log(x_n + \nu^2) \right\} + \lambda \Psi(\bm{D}\bm{x}),
\end{equation}
where $\bm{y} \in \mathbb{R}_{>0}^N$ is the observed signal, $\nu^2$ represents the thermal noise variance, and $\bm{D}$ is the first-order difference operator.
The regularization term $\Psi(\bm{D}\bm{x})$ promotes block-sparsity in the gradient domain, effectively modeling the piecewise linear structure (slopes) of the signal without introducing staircase artifacts common in standard TV regularization.
One of the advantages of our proposed framework is its compatibility with general data fidelity terms.
The proximal operator for the Shifted I-divergence has a closed-form solution~\cite[Lemma 2.5 (ii) and Example 4.9]{chauxVariationalFormulationFramebased2007}, enabling efficient optimization via the proposed ADMM-based algorithms (\zcref{alg:relopl2l1, alg:adalopl2l1}).

\subsubsection{Experimental Setup}
We generated synthetic nanopore signals of length $N=1024$ simulating random blockade events with depths uniformly distributed in $[30, 100]$ pA and dwell times following a Gamma distribution $\Gamma(2, 20)$, on a baseline of 150 pA.
The signals were filtered using a 4th-order Bessel filter with a time constant $\tau_{\text{cap}} = 4$ to simulate sensor capacitance.
MPG noise was added with a shot noise scale $\alpha_{\text{shot}} = 20$ and thermal noise $\sigma_{\text{thermal}} = 5$.
Examples of the generated nanopore signals are shown in \zcref{fig:nanopore_samples}.
We compared the proposed LogLOP-\ltwolone and AdaLOP-\ltwolone (using SID) against LOP-\ltwolone (using both L2 and SID) and GME-LOP-\ltwolone (L2 only).
We performed 500 independent trials.

\begin{figure}[t]
    \centering
    \includegraphics[width=0.85\linewidth]{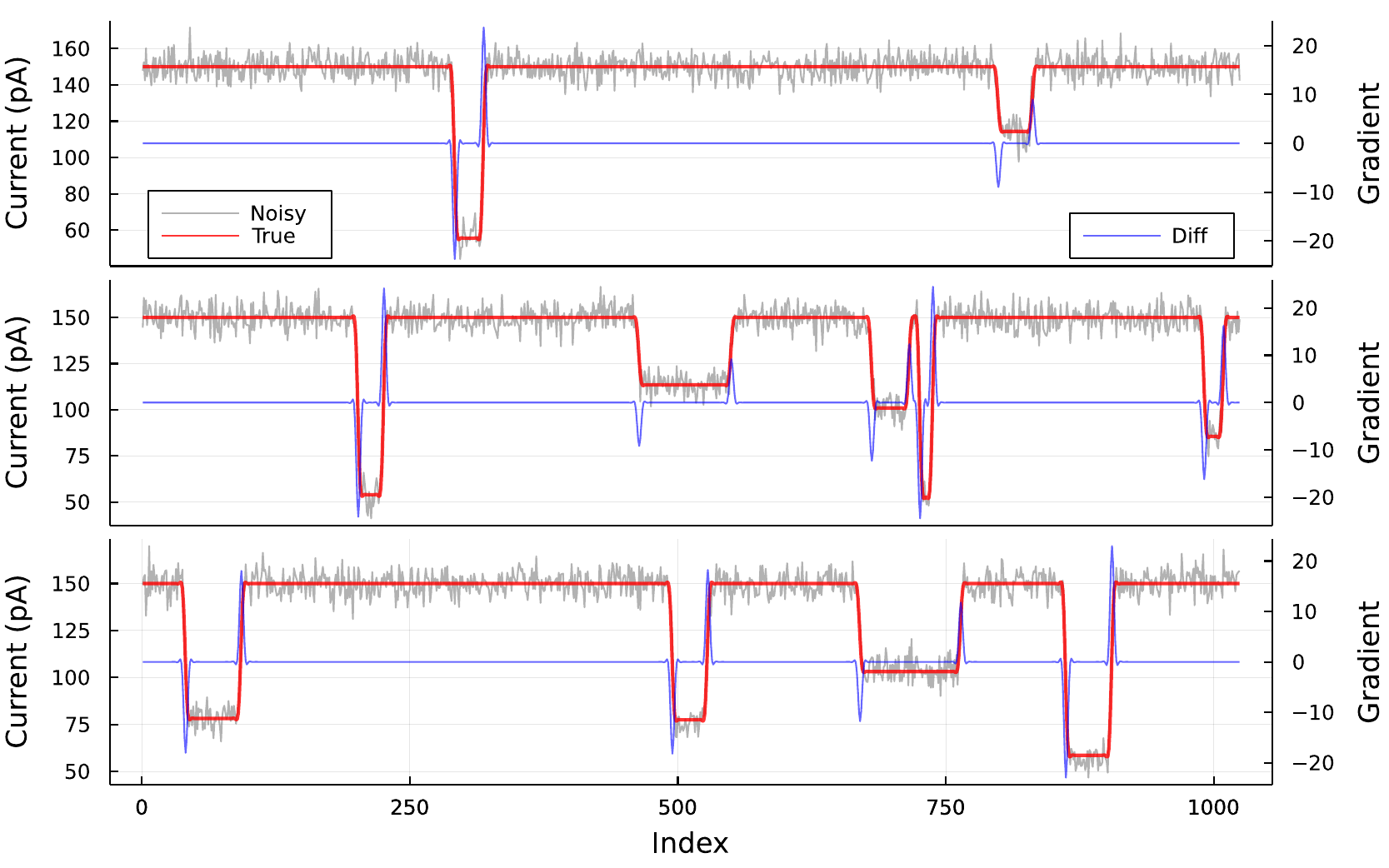}
    \caption{Examples of synthetic nanopore current signals. The signals exhibit piecewise constant structures with transient slopes (i.e., block-sparse gradients) and are corrupted by Mixed Poisson-Gaussian noise.}
    \zlabel{fig:nanopore_samples}
\end{figure}

\subsubsection{Results}
\begin{table}[t]
    \centering
    \caption{Denoising performance (SNR (dB)) on synthetic nanopore signals over 500 independent trials.}
    \zlabel{tab:nanopore_results}
    \begin{tabular}{lccc}
        \hline
        Method                 & Mean           & Median         & Std Dev \\
        \hline
        LOP-\ltwolone (L2)     & 28.41          & 28.39          & 0.40    \\
        GME-LOP-\ltwolone (L2) & 27.89          & 27.89          & 0.32    \\
        LOP-\ltwolone (SID)    & 32.77          & 32.24          & 2.55    \\
        LogLOP-\ltwolone (SID) & 32.02          & 31.76          & 1.16    \\
        AdaLOP-\ltwolone (SID) & \textbf{33.41} & \textbf{33.17} & 1.88    \\
        \hline
    \end{tabular}
\end{table}

\zcref{tab:nanopore_results} summarizes the denoising performance in terms of SNR.
The methods using the SID significantly outperform those using the standard squared error loss (L2), confirming the importance of appropriate noise modeling.
Among the SID-based methods, the proposed AdaLOP-\ltwolone achieves the highest mean SNR of 33.41 dB.
LogLOP-\ltwolone also shows robust performance with 32.02 dB.
These results demonstrate that the proposed nonconvex penalties effectively recover the underlying signal in the presence of MPG noise.

\section{Conclusion}
\zlabel{sec:conclusion}
We proposed LogLOP-\ltwolone and AdaLOP-\ltwolone, two nonconvex regularization frameworks for block-sparse recovery that do not require prior knowledge of block partitions.
These methods approximate the $\ell_0$ pseudo-norm via nonconvex variational functions to mitigate underestimation bias, while offering flexibility to handle diverse noise models beyond Gaussian assumptions.
We demonstrated their superior performance over state-of-the-art baselines in both synthetic and real-world signal processing tasks.
Although theoretical global convergence remains open due to nonconvexity, our ADMM-based algorithms exhibit stable empirical convergence.
A key limitation is the reliance on empirical hyperparameter selection. Developing automated strategies to optimally determine parameters such as the block partitions granularity remains an important direction for future research.

\printbibliography
\appendix
\appendices

\section{Existence of Minimizer for General LOP Penalty}
\zlabel{app:existence_min_adalop}

We verify the existence of a minimizer for the joint variables $(\bm x, \bm \sigma)$ in the general LOP-type penalty \eqref{eq:general-lop-framework}.
Consider the function $H(\bm x, \bm \sigma) = h(\bm x, \bm \sigma) + \iota_{\tilde{C}}(\bm x, \bm \sigma)$. Here, $h(\bm x, \bm \sigma) \coloneqq \sum_{n=1}^N \varphi(x_n, \sigma_n)$ is composed of a proper lsc function $\varphi$ that satisfies the \emph{asymptotically level stable} (als) property~\cite[Def. 3.3.2]{alfredauslenderAsymptoticConesFunctions2003}. The constraint set is $\tilde{C} \coloneqq \mathbb{R}^N \times C$, where $C = \{ \bm \sigma \in \mathbb{R}^N \mid \|\bm{D\sigma}\|_1 \le \alpha \}$ is a polyhedral set.

\textbf{Constraint Set Properties:} The constraint set $\tilde{C}$ is \emph{asymptotically linear}~\cite[Def. 2.3.2 and Prop. 2.3.1]{alfredauslenderAsymptoticConesFunctions2003} because it is a polyhedral set~\cite[Sec. 2.2.4]{boydConvexOptimization2023}.

\textbf{Variational Function Properties:} We now verify the als property for the separable sum $h$, assuming it holds for each $\varphi$.
Consider a sequence $\{ \bm z^k \}$ with $\|\bm z^k\| \to \infty$ and $\bm z^k / \|\bm z^k\| \to \bar{\bm z} \in \mathbb{R}^{2N}$.
Decompose $\bm z^k$ into $N$ pairs $\{ (z_n^k) \}_{n=1}^N$ where $z_n$ consists of $(x_n, \sigma_n)$.
For divergent components, the als property of $\varphi$ ensures $\varphi(z_n^k - \rho \bar{z}_n) \le \varphi(z_n^k)$ for any $\rho > 0$ and sufficiently large $k$.
For bounded components, $\bar{z}_n = \lim_{k\to\infty} z_n^k/\|\bm z^k\| = 0$, leading to $\varphi(z_n^k - \rho \bar{z}_n) = \varphi(z_n^k)$.

Summing these inequalities, we have
\begin{equation*}
    h(\bm z^k - \rho \bar{\bm z}) = \sum_{n=1}^N \varphi(z_n^k - \rho \bar{z}_n) \le \sum_{n=1}^N \varphi(z_n^k) = h(\bm z^k).
\end{equation*}
Thus, if $\bm z^k$ is in a level set of $h$, then $\bm z^k - \rho \bar{\bm z}$ is also in the same level set, verifying the als property of $h$.

Consequently, the sum $H = h + \iota_{\tilde{C}}$ is als~\cite[Prop. 3.3.3(c)]{alfredauslenderAsymptoticConesFunctions2003}.
Because $H$ is lower-bounded,~\cite[Cor. 3.4.2]{alfredauslenderAsymptoticConesFunctions2003} guarantees that a minimizer of $H$ exists.
This result extends to the regularized problem \zcref{eq:regularized_objective} where the data fidelity $f$ is proper, lsc, coercive, and convex. The extension is valid because $f \circ \bm A$ is weakly coercive~\cite[Example 3.2.1]{alfredauslenderAsymptoticConesFunctions2003}, and the sum $f \circ \bm A + H$ remains als by combining the definitions of weakly coercivity and als~\cite[Def. 3.2.1 and Def 3.3.2]{alfredauslenderAsymptoticConesFunctions2003} with the sum rule for asymptotic functions~\cite[Prop. 2.6.1(b)]{alfredauslenderAsymptoticConesFunctions2003}.

For specific penalties, LogLOP-\ltwolone satisfies the als property because its variational function $\phi_{\epsilon}$ (see \zcref{eq:var-rwl1}) is naturally coercive.
AdaLOP-\ltwolone also satisfies the als property.
For $\phi_{\gamma, w} = g \circ \phi$, the als property follows directly from the geometric properties of $\phi$. Since $\phi$ is convex and coercive~\cite[Lem. 1]{9729560}, its level sets are bounded and it increases monotonically along any radial direction.
This ensures $\phi(\bm z^k - \rho \bar{\bm z}) \le \phi(\bm z^k)$ for any sequence $\{ \bm z^k \}$ diverging to infinity in direction $\bar{\bm z}$ and sufficiently large $k$.
The non-decreasing nature of $g$ preserves this inequality for $\phi_{\gamma, w}$, thereby verifying the als property.

\section{Proof of \zcref{thm:var-rwl1}}
\zlabel{app:var-rwl1}
We analyze the behavior of $\Psi_{\alpha,\epsilon}(\bm x)$ in the limits of $\alpha$.
As $\alpha\to\infty$, the constraint $\norm{\bm{D\sigma}}_1 \leq \alpha$ becomes inactive. By \zcref{lem:var-rwl1}, the term $\sum_{n=1}^N\phi_{\epsilon}(x_n,\sigma_n)$ is minimized when $\sigma_n = (\abs{x_n}/\epsilon+1)^2$, implying $\lim_{\alpha\to\infty}\Psi_{\alpha,\epsilon}(\bm x) = \Omega_{\epsilon}(\bm x)$.

Conversely, as $\alpha\to0$, the constraint enforces $\norm{\bm{D\sigma}}_1 = 0$, implying $\sigma_n = \tau$ for all $n$ and some $\tau \geq 1$. Thus:
\begin{align*}
    \Psi_{0,\epsilon}(\bm x)
     & = \min_{\tau \geq 1}\sum_{n=1}^N \qty{\frac1{2\tau}\qty(\frac{\abs{x_n}}{\epsilon}+1)^2 + \frac12\log\tau - \frac12}                       \\
     & = \min_{\tau \geq 1} \left\{ \frac{1}{2\tau}\sum_{n=1}^N \qty(\frac{\abs{x_n}}{\epsilon}+1)^2 + \frac{N}{2}\log\tau - \frac{N}{2} \right\} \\
     & = \frac{N}{2} + N\log\sqrt{\frac1N{\sum_{n=1}^N \qty(\frac{\abs{x_n}   }{\epsilon}+1)^2}} - \frac{N}{2}                                    \\
     & = N\log\sqrt{\frac1N{\sum_{n=1}^N \qty(\frac{\abs{x_n}}{\epsilon}+1)^2}}.
\end{align*}
The transition from the second to the third line follows by substituting the optimal $\tau = N^{-1}\sum_{n=1}^N \qty({\abs{x_n}}/{\epsilon}+1)^2$, which can be derived similarly to \zcref{lem:var-rwl1}.

\section{Proof of \zcref{prop:monotonicity-of-f}}
\zlabel{app:monotonicity-of-f}
We show that the penalty difference $\Delta(z)$ is positive for $z > L$ and $M > L$, where $M = \abs{\mathcal B}$, $L = \abs{\mathcal B'}$, and $z \coloneqq \sum_{n \in \mathcal B'} (\abs{x_n}/\epsilon+1)^2$.
The penalty difference is:
\begin{equation}
    \Delta(z) \coloneqq M \log\left(\frac{z+M-L}{M}\right) - L \log\left(\frac{z}{L}\right).
\end{equation}
Note that $\Delta(L) = 0$.
The derivative with respect to $z$ is:
\begin{align}
    \Delta'(z) & = \frac{M}{z+M-L} - \frac{L}{z} = \frac{(M-L)(z-L)}{z(z+M-L)}.
\end{align}
For $z > L$ and $M > L$, we have $M-L > 0$, $z-L > 0$, and $z(z+M-L) > 0$. Thus, $\Delta'(z) > 0$.
Because $\Delta(L) = 0$ and $\Delta(z)$ is strictly increasing for $z > L$, it follows that $\Delta(z) > 0$.
This confirms that separating the zero block $\mathcal B''$ from the block $\mathcal B$ strictly reduces the penalty.

\section{Proof of \zcref{lem:var-ada-lop-l2l1}}
\zlabel{app:proof_adalopl2l1_upper_bounds_mcp}

We establish the inequality $\phi_{\gamma, w}(x, \sigma) \geq \Omega_{\gamma, w}(x)$ for any $w, \gamma > 0$.
The MCP function $g \coloneqq \Omega_{\gamma, w}$ is extended to the extended real line $[0, \infty]$ as
\begin{equation}
    g(z) = \begin{cases}
        w z - \frac{z^2}{2\gamma}, & \text{if } 0 \leq z \leq \gamma w; \\
        \frac{w^2 \gamma}{2},      & \text{if } z \geq \gamma w.
    \end{cases}
    \label{eq:mcp-def}
\end{equation}
Note that $g(\abs{x}) = \Omega_{\gamma, w}(x)$ for all $x \in \mathbb{R}$.
The continuity and monotonicity of $g$ on $[0, \infty]$ ensure that $\phi_{\gamma, w} = g \circ \phi$ is well-defined and satisfies the required inequality.

\textbf{Monotonicity of $g(z)$.}
For $0 \le z < \gamma w$, $g'(z) = w - z/\gamma > 0$, while for $z \ge \gamma w$, $g(z)$ remains constant at $w^2\gamma/2$, implying $g'(z) = 0$ for $z > \gamma w$.
Continuity at $z = \gamma w$ is maintained as $\lim_{z \uparrow \gamma w} g(z) = \frac{w^2\gamma}{2} = g(\gamma w)$.
Thus, $g(z)$ is nondecreasing on the extended real line $[0, \infty]$.

Since $\phi(x, \sigma) \ge \abs{x}$ and $g(z)$ is nondecreasing:
\begin{equation}
    g(\phi(x, \sigma)) \ge g(\abs{x}) \implies \phi_{\gamma, w}(x, \sigma) \ge \Omega_{\gamma, w}(x).
\end{equation}
Equality holds when $\tau = \abs{x}$~\cite[Lemma 1]{9729560}.

\section{Proof of \zcref{thm:ada-lop-l2l1}}
\zlabel{app:ada-lop-l2l1}
We analyze the limits of $\Psi_{\alpha, \gamma,\bm w}(\bm x)$.
As $\alpha\to\infty$, the TV constraint $\norm{\bm{D\sigma}}_1\leq\alpha$ is ignored. Minimizing $\sum_{n=1}^N$ $\{\omega_n\cdot$ $\phi(x_n,\sigma_n)\}$ yields $\sigma_n = \abs{x_n}$, so $\Psi_{\alpha, \gamma,\bm w}(\bm x) \to \Omega_{\gamma,\bm w}(\bm x)$.

As $\alpha\to 0$, the constraint implies $\sigma_n = \tau$ for all $n$. Thus:
\begin{align}
    \Psi_{0,\gamma,\bm w}(\bm x) & = \min_{\substack{\tau\in\R_{\geq0} \\ \bm\omega\in\R_{\geq0}}}\sum_{n=1}^N\qty{\omega_n\qty(\frac{\abs{x_n}^2}{2\tau}+\frac\tau2)} + \frac{\gamma}{2} \norm{\bm\omega - \bm w}_2^2 \\
                                 & = \min_{\substack{\tau\in\R_{\geq0} \\ \bm\omega\in\R_{\geq0}}}\frac{\norm{\bm x}_{\bm\omega,2}^2}{2\tau} + \frac{\tau}{2}\norm{\bm\omega}_1 + \frac{\gamma}{2} \norm{\bm\omega - \bm w}_2^2 .
\end{align}
Minimizing over $\tau$ yields $\tau = \norm{\bm x}_{\bm\omega,2}/ \norm{\bm\omega}_1^{1/2}$ (similar to the proof of~\cite[Lemma 1]{9729560}). Substituting this back yields \zcref{eq:ada-l2}.

\section{Derivation of \zcref{alg:relopl2l1}}
\zlabel{app:relopl2l1-alg}
We derive the ADMM updates for the problem regularized by the LogLOP-$\ell_{2}$/$\ell_{1}$ penalty:
\begin{equation}
    \min_{\bm x \in \R^N} f(\bm{Ax}) + \lambda\Psi_{\alpha,\epsilon}(\bm{Lx}).
\end{equation}
Introducing auxiliary variables $\bm u = \bm{Ax}$, $\bm v = \bm{Lx}$, $\bm\eta = \bm{D\sigma}$, and $\bm\xi = \bm\sigma$, we rewrite the problem as:
\begin{align}
     & \min_{\substack{\bm x \in \R^N, \bm\sigma \in \R^K, \bm u \in \R^J,                               \\ \bm v \in \R^K, \bm\eta \in \R^{K-1}, \bm\xi \in \R^K}} f(\bm{u}) + \lambda\sum_{n=1}^N \phi_{\epsilon}(v_n,\xi_n) + \iota_{B_1^{\alpha}}(\bm{\eta}),                                                 \\
     & \quad \text{ s.t. } \bm u = \bm{Ax}, \bm v = \bm{Lx}, \bm\eta = \bm{D\sigma}, \bm\xi = \bm\sigma.
\end{align}
Here, $\iota_C(\bm x)$ is the indicator function of a set $C$, which takes $0$ if $\bm x \in C$ and $\infty$ otherwise.
The augmented Lagrangian is:
\begin{align}
    \label{eq:relopl2l1-augmented-lagrange}
    \mathcal{L} & = f(\bm u) + \lambda\sum_{n=1}^N \phi_\epsilon(v_n, \xi_n) + \iota_{B_1^{\alpha}}(\bm\eta)            \\
                & + \bm r_1^\top(\bm u-\bm{Ax}) + \frac{\mu_1}{2}\norm{\bm u-\bm{Ax}}_2^2 + \bm r_2^\top(\bm v-\bm{Lx}) \\ &+ \frac{\mu_2}{2}\norm{\bm v-\bm{Lx}}_2^2
    + \bm r_3^\top(\bm\eta-\bm{D\sigma})+ \frac{\mu_3}{2}\norm{\bm\eta-\bm{D\sigma}}_2^2                                \\
                & + \bm r_4^\top(\bm\xi-\bm\sigma)+ \frac{\mu_4}{2}\norm{\bm\xi-\bm\sigma}_2^2,
\end{align}
where $\mu_i > 0$ are ADMM parameters and $\bm r_i$ are Lagrange multipliers. \zcref{alg:relopl2l1} monotonically decreases $\mathcal{L}$ provided $\mu_4 \geq \lambda/54$ because the ADMM update rule is derived from alternating minimization and the subproblems are convex.

1) Updates for $\bm x, \bm u, \bm v, \bm \sigma, \bm \eta$: Minimizing $\mathcal L$ with respect to each variable yields the updates derived in \zcref{eq:x-update,eq:u-update,eq:v-update-relop,eq:sigma-update,eq:eta-update}.

2) $\bm\xi$-update: The subproblem separates into element-wise minimizations as demonstrated in \zcref{eq:xi-update-relop}.
To ensure strict convexity and uniqueness of the solution, we require $\mu_4 > \lambda/54$.
This condition is derived from a lower bound of the second derivative with respect to $\xi_n$:
\begin{align}
    \frac{\partial^2 \mathcal L}{\partial \xi_n^2} & = \lambda\qty(\frac{a}{\xi_n^3} - \frac{1}{2\xi_n^2}) + \mu_4 \geq \mu_4 - \frac{\lambda}{54a^2} \geq \mu_4 - \frac{\lambda}{54}.
\end{align}
where $a = (\abs{v_n}/\epsilon+1)^2$. We solve for $\xi_n$ with the bisection method. The first derivative decides the bisection interval:
\begin{align}
    \frac{\partial \mathcal L}{\partial \xi_n} & = \frac{\lambda}{2\xi_n^2}\qty(\xi_n - a) + \mu_4(\xi_n-b),
\end{align}
where $b = \sigma_n - \mu_4^{-1}r_{4,n}$.
By substituting $\xi_n=a$ and $\xi_n=b$ into the derivative, we obtain $\frac{\partial \mathcal L}{\partial \xi_n}\mid_{\xi_n=a} = \mu_4(a-b)$ and $\frac{\partial \mathcal L}{\partial \xi_n}\mid_{\xi_n=b} = -\frac{\lambda}{2b^2}(a-b)$, which have opposite signs. Thus, the solution lies in $[\min(a,b), \max(a,b)]$. Considering the constraint $\xi_n \geq 1$ and the fact that $a \geq 1$, the optimal solution is bracketed by $\xi_n = \max\qty{1, \min\qty{a, b}}$ and $\xi_n = \max\qty{a, b}$.

3) Dual variable $(\bm r_i)_{i=1}^4$ updates are described in \zcref{eq:dual-update}.

\section{Derivation of \zcref{alg:adalopl2l1}}
\zlabel{app:adalopl2l1-alg}

We derive the ADMM updates for the problem regularized by the AdaLOP-$\ell_{2}$/$\ell_{1}$ penalty:
\begin{equation}
    \min_{\bm x \in \R^N} f(\bm{Ax}) + \Psi_{\alpha, \gamma, \bm w}(\bm{Lx}).
\end{equation}
The augmented Lagrangian $\mathcal{L}$ is:
\begin{align}
    \label{eq:adalopl2l1-augmented-lagrange}
    \mathcal{L} & = f(\bm u) + \sum_{n=1}^N \omega_n \phi(v_n, \xi_n) + \frac{\gamma}{2}\norm{\bm\omega - \bm w}_2^2 + \iota_{B_1^{\alpha}}(\bm\eta) \\
                & + \bm r_1^\top(\bm u - \bm{Ax}) + \frac{\mu_1}{2}\norm{\bm u - \bm{Ax}}_2^2 + \bm r_2^\top(\bm v - \bm{Lx})                        \\
                & + \frac{\mu_2}{2}\norm{\bm v - \bm{Lx}}_2^2
    + \bm r_3^\top(\bm\eta - \bm{D\sigma}) + \frac{\mu_3}{2}\norm{\bm\eta - \bm{D\sigma}}_2^2                                                        \\
                & + \bm r_4^\top(\bm\xi - \bm\sigma) + \frac{\mu_4}{2}\norm{\bm\xi - \bm\sigma}_2^2.
\end{align}
\zcref{alg:adalopl2l1} monotonically decreases $\mathcal{L}$.

1) Updates for $\bm x, \bm u, \bm \sigma, \bm \eta, (\bm r_i)_{i=1}^4$: These are identical to the LogLOP-\ltwolone updates as described in \zcref{eq:x-update,eq:u-update,eq:sigma-update,eq:eta-update}.

2) Joint $(\bm v, \bm\xi)$-update: Setting $\mu_4 = \mu_2$, the update is the proximal operator of $\phi$ as described in \zcref{eq:v-xi-update-adalop}.

3) $\bm \omega$-update: Minimizing over $\omega_n \geq 0$ leads to a solution as shown in \zcref{eq:omega-update}.
While the rule $0 \cdot (-\infty) = -\infty$ is theoretically required to handle $\phi(v_n, \xi_n) = \infty$, initializing with a feasible point ensures $\phi(v_n,\tau_n) < \infty$ at every iteration in practice.

4) $\bm w$-update: The minimum occurs when $\bm w \leftarrow \bm \omega$.

\section{Kurdyka-\L ojasiewicz Property of $\phi(x, \tau)$}
\zlabel{app:kl-phi}
We show that $\phi: \mathbb{R} \times \mathbb{R} \to \mathbb{R}_{\geq 0} \cup \{+\infty\}$ given in \zcref{eq:var-l2norm} satisfies the Kurdyka-\L ojasiewicz (KL) property.
A function has the KL property if it is proper, lower semicontinuous, and its graph is a semialgebraic set (or definable in an o-minimal structure). Since $\dom(\phi) \neq \emptyset$ and $\phi(x, \tau) \geq 0$, it is proper. The lower semicontinuity of $\phi$ is proved in~\cite[App. B]{9729560}.

The graph of $\phi$ is $\text{graph}(\phi) = \{(x,\tau,y) \in \mathbb{R}^3 \mid y = \phi(x,\tau) \text{ and } \phi(x,\tau) < \infty \}$.
\begin{itemize}[leftmargin=*]
    \item \textbf{Case 1: $\tau > 0$}.
          The condition $y = \frac{x^2}{2\tau} + \frac{\tau}{2}$ is equivalent to $2\tau y - x^2 - \tau^2 = 0$.
          The set $G_1 = \{ (x,\tau,y) \in \mathbb{R}^3 \mid 2\tau y - x^2 - \tau^2 = 0 \land \tau > 0 \}$ is semialgebraic as it is defined by polynomial equality and inequality.

    \item \textbf{Case 2: $x = 0$ and $\tau = 0$}.
          The set $G_2 = \{ (x,\tau,y) \in \mathbb{R}^3 \mid x=0 \land \tau=0 \land y = 0 \}$ is semialgebraic.
\end{itemize}
Since $\text{graph}(\phi) = G_1 \cup G_2$ is a finite union of semialgebraic sets, it is semialgebraic.
Being proper, lsc, and having a semialgebraic graph, $\phi(x, \tau)$ possesses the KL property.

\section{Kurdyka-\L ojasiewicz Property of $\phi_\epsilon(x, \tau)$}
\zlabel{app:kl-phi-epsilon}
\newcommand{\Ranexp}{\mathbb{R}_{\text{an, exp}}}
\newcommand{\Ralg}{\mathbb{R}_{\text{alg}}}

We verify the KL property for the function $\phi_\epsilon(x, \tau)$.

\textbf{Properness:}
The $\dom(\phi_\epsilon) = \{(x,\tau) \mid \tau \geq 1\}$ is nonempty.
For $\tau \geq 1$, $\phi_\epsilon(x, \tau) \geq \log \left(|x|/\epsilon + 1 \right) \geq 0$. Thus, $\phi_\epsilon$ is proper.

\textbf{Lower semicontinuity:}
We show that the level sets $L_a = \{(x, \tau) \in \mathbb{R}^2 \mid \phi_\epsilon(x, \tau) \le a\}$ are closed for any $a \in \mathbb{R}$.
Since $\dom(\phi_\epsilon) = \{(x, \tau) \mid \tau \ge 1\}$, points with $\tau < 1$ have $\phi_\epsilon(x, \tau) = \infty$, so they are not in $L_a$. Thus, $L_a \subseteq \{(x, \tau) \mid \tau \ge 1\}$.
On the set $\{(x, \tau) \mid \tau \ge 1\}$, $\phi_\epsilon(x, \tau)$ is continuous. Therefore, $L_a$ is the intersection of the closed set $\{(x, \tau) \mid \tau \ge 1\}$ and the preimage of the closed set $(-\infty, a]$ under a continuous function. Thus, $L_a$ is closed, which implies that $\phi_\epsilon$ is lower semicontinuous.

\textbf{Definability:}
For $\tau \geq 1$, $\phi_\epsilon(x, \tau)$ involves $|x|$, $1/\tau$, and $\log \tau$, which are definable in the log-exp structure $\Ranexp$~\cite{attouchProximalAlternatingMinimization2010}.
The graph restricted to $\tau \geq 1$ is thus definable in $\Ranexp$.
Because $\phi_\epsilon(x, \tau)$ is proper, lsc, and definable in $\Ranexp$, it satisfies the KL property.

\section{Kurdyka-\L ojasiewicz Property of $\psi(x, \tau, \omega, w)$}
\zlabel{app:kl-psi}

We establish the KL property for the joint function $\psi(x, \tau, \omega, w) \coloneqq \omega \phi(x, \tau) + \frac{\gamma}{2}(\omega - w)^2$, which appears in the minimization problem of AdaLOP-\ltwolone.
The variational function $\phi_{\gamma, w}(x, \tau)$ satisfies the KL property (see \zcref{app:kl-phi-gamma-w}).
However, \zcref{alg:adalopl2l1} operates on the augmented Lagrangian involving the joint variables $(x, \tau, \omega, w)$. Therefore, we verify the KL property for $\psi$ with respect to the joint variables.

\textbf{Properness:}
The effective domain of $\psi$ is $\dom(\psi) = \dom(\phi) \times \mathbb{R}_{\geq 0} \times \mathbb{R}_{\geq 0}$. Because $\dom(\phi)$ is nonempty, $\dom(\psi)$ is also nonempty.
Furthermore, as $\phi(x, \tau) \ge 0$, $\omega \ge 0$, and $\gamma > 0$, the function is bounded below by $0$. Thus, $\psi$ is proper.

\textbf{Lower semicontinuity (lsc):}
The function $\phi(x, \tau)$ is lsc as shown in~\cite[App. B]{9729560}.
The variable $\omega$ is nonnegative, and the mapping $(x, \tau, \omega, w) \mapsto \omega$ is continuous.
Because the product of nonnegative lsc functions is lsc whenever it is defined~\cite{bourbakiGeneralTopology1989}, the term $\omega \phi(x, \tau)$ is lsc under the rule $0 \cdot \infty = 0$.
The quadratic term $\frac{\gamma}{2}(\omega - w)^2$ is clearly continuous.
The sum of lsc functions is lsc~\cite{bourbakiGeneralTopology1989}, hence $\psi$ is lower semicontinuous.

\textbf{Definability:}
The graph of $\phi(x, \tau)$ is semialgebraic, as proven in \zcref{app:kl-phi}.
Polynomial functions such as $\omega$ and $(\omega - w)^2$, are semialgebraic.
Because the class of semialgebraic functions is closed under finite sums and products~\cite[Sec. 4.3]{attouchProximalAlternatingMinimization2010}, $\psi(x, \tau, \omega, w)$ is semialgebraic.
Consequently, $\psi(x, \tau, \omega, w)$ is definable in an o-minimal structure.

Being proper, lsc, and definable, the joint variational function $\psi(x, \tau, \omega, w)$ satisfies the KL property.


\section{Kurdyka-\L ojasiewicz Property of $\phi_{\gamma, w}(x, \tau)$}
\zlabel{app:kl-phi-gamma-w}

We verify the KL property for the composite function $\phi_{\gamma, w}(x, \tau) = g(\phi(x, \tau))$ defined in \zcref{eq:eq-var-lop-l2l1}.
Note that $g(z)$ is the nondecreasing, piecewise polynomial function \zcref{eq:mcp-def} and $\phi(x, \tau)$ is the KL function (see \zcref{app:kl-phi}).

\textbf{Properness:}
Since $g$ is defined on the extended real line $[0, \infty]$ (see \zcref{app:proof_adalopl2l1_upper_bounds_mcp}), the composition $\phi_{\gamma, w} = g \circ \phi$ is well-defined for all $(x, \tau)$.
Given that $0 \le \phi_{\gamma, w}(x, \tau) \le \frac{\gamma w^2}{2}$, it follows that $\phi_{\gamma, w}$ is proper.

\textbf{Lower semicontinuity:}
We show that the level sets $L_a = \{ (x, \tau) \mid \phi_{\gamma, w}(x, \tau) \leq a \}$ are closed for any $a \in \mathbb{R}$.
\begin{enumerate}
    \item If $a \geq \frac{\gamma w^2}{2}$, the set $L_a$ is $\mathbb{R}^2$, which is closed.
    \item If $a < \frac{\gamma w^2}{2}$, the condition $\phi_{\gamma, w}(x, \tau) \leq a$ implies $g(\phi(x, \tau)) \leq a$. In this range, $g(z)$ is strictly increasing. Thus, this is equivalent to $\phi(x, \tau) \leq g^{-1}(a)$. Since $\phi(x, \tau)$ is lower semicontinuous, its level set is closed.
\end{enumerate}
Thus, $\phi_{\gamma, w}$ is lower semicontinuous.

\textbf{Definability:}
The composition of selfalgebraic functions is selfalgebraic~\cite{attouchProximalAlternatingMinimization2010}.
Therefore, $\phi_{\gamma, w} = g \circ \phi$ is selfalgebraic and hence definable~\cite[Sect. 4.3]{attouchProximalAlternatingMinimization2010}.

Consequently, $\phi_{\gamma, w}$ satisfies the KL property.


\end{document}